%% file: generalization.tex
\documentclass{article}
\usepackage[preprint]{neurips_2023}
\usepackage{xr}
\usepackage[utf8]{inputenc} 
\usepackage[T1]{fontenc}    
\usepackage{hyperref}       
\usepackage{url}            
\usepackage{booktabs}       
\usepackage{amsfonts}       
\usepackage{nicefrac}       
\usepackage{microtype}      

\usepackage{graphicx}
\usepackage{subcaption}
\usepackage{booktabs} 
\usepackage{amsmath}
\usepackage{amsthm}
\usepackage{amssymb}
\usepackage{mathtools}
\usepackage{natbib}
\usepackage{algorithm}
\usepackage{algorithmic}
\usepackage{thmtools}
\usepackage{thm-restate}
\usepackage{multirow}

\usepackage{hyperref}
\usepackage{xcolor}

\newcommand{\bz}{\boldsymbol{z}}
\newcommand{\bx}{\boldsymbol{x}}

\newcommand{\bxi}{\boldsymbol{\xi}}

\newcommand{\bs}{\boldsymbol{s}}

\newcommand{\by}{\boldsymbol{y}}

\newcommand{\bS}{\boldsymbol{S}}

\newcommand{\beps}{\boldsymbol{\epsilon}}
\newcommand{\btheta}{\boldsymbol{\theta}}
\newcommand{\bmu}{\boldsymbol{\mu}}

\newcommand{\bI}{\boldsymbol{I}}

\newcommand{\bbP}{\mathbb{P}}

\newcommand{\bbR}{\mathbb{R}}

\newcommand{\bv}{\boldsymbol{v}}

\DeclareMathOperator*{\argmin}{arg\,min}

\newtheorem{definition}{\textbf{Definition}}

\newtheorem{lemma}{\textbf{Lemma}}

\newtheorem{proposition}{\textbf{Proposition}}
\newtheorem{remark}{\textbf{Remark}}

\newtheorem{example}{\textbf{Example}}

\newcommand{\mE}{\mathbb{E}}

\newcommand{\cX}{\mathcal{X}}

\newcommand{\cC}{\mathcal{C}}
\newcommand{\cN}{\mathcal{N}}

\newcommand{\cF}{\mathcal{F}}

\newcommand{\cO}{\mathcal{O}}

\makeatother

\title{On the Generalization of Diffusion Model}
\begin{document}
	\date{}
	\author{Mingyang Yi$^{1}$, Jiacheng Sun$^{1}$, Zhenguo Li$^{1}$\\
		$^{1}$Huawei Noah’s Ark Lab \\
		\texttt{\{yimingyang2,sunjiacheng1,li.zhenguo\}@huawei.com}
	}
	\maketitle
	\begin{abstract}
		The diffusion probabilistic generative models are widely used to generate high-quality data. Though they can synthetic data that does not exist in the training set, the rationale behind such generalization is still unexplored. In this paper, we formally define the generalization of the generative model, which is measured by the mutual information between the generated data and the training set. The definition originates from the intuition that the model which generates data with less correlation to the training set exhibits better generalization ability. Meanwhile, we show that for the empirical optimal diffusion model, the data generated by a deterministic sampler are all highly related to the training set, thus poor generalization. This result contradicts the observation of the trained diffusion model's (approximating empirical optima) extrapolation ability (generating unseen data). To understand this contradiction, we empirically verify the difference between the sufficiently trained diffusion model and the empirical optima. We found, though obtained through sufficient training, there still exists a slight difference between them, which is critical to making the diffusion model generalizable. Moreover, we propose another training objective whose empirical optimal solution has no potential generalization problem. We empirically show that the proposed training objective returns a similar model to the original one, which further verifies the generalization ability of the trained diffusion model.      
	\end{abstract}
	\section{Introduction}\label{sec:intro}
	The technique of generative model is capable of synthetic data from the target distribution, which has been well-developed in recent years e.g., VAE \citep{kingma2013auto}, GAN \citep{goodfellow2014generative}, and denoise diffusion probabilistic model (DDPM) \citep{song2020score,ho2020denoising} etc. Among all these methods, the diffusion model has recently attracted great attention due to its capability of generating high-quality data that does not exist in the training set. However, some recent works \citep{somepalli2022diffusion,carlini2023extracting} have empirically shown that the diffusion model tends to generate data that is combined with the parts of data in the training set. This phenomenon threatens the application of the diffusion model in the aspect of privacy, as it may leak user's data \citep{carlini2023extracting}. 
	\par
	Ideally, the generative model should be capable of generating data from the underlying target distribution, but with less dependence on training data (so that extrapolating). Inspired by this intuition, we define the excess risk of the generative model which measures its performance of it. In contrast to the existing literature \citep{goodfellow2014generative,arjovsky2017wasserstein,ho2020denoising,song2020score}, which only focuses on the quality of generated data, the defined excess risk also considers the generalization of the model. Concretely, our excess risk can be decomposed as the optimization error and the generalization error. The optimization error is explained as a distance between the distribution of generated data and the target one, which is the most commonly used metric to evaluate the generative model \citep{kingma2013auto}. On the other hand, the generalization error cares about the ``extrapolation'' of the model, which intuitively is the correlation between generated data and the training set. Owing to this, the generalization error is defined as the mutual information \citep{duchi2016lecture} between them.
	\par
	With the defined excess risk to measure the performance of the generative model, we apply it to check the quality of the diffusion model. As the model is trained by minimizing an empirical noise prediction problem \citep{song2020score,ho2020denoising}, we first analyze its empirical optimal solution. We show the solution can converge to the one with guaranteed optimization error. However, due to the formulation of the solution, generating data with deterministic update rule \citep{song2020denoising,lu2022dpm} will generate data highly related to the training set, which results in poor generalization. Thus, as the sufficiently trained neural network can converge to the global minima of training objective \citep{allen2019convergence,du2019gradient}, we are motivated to explore whether the poor generalization transfers to the well-trained diffusion model.  
	\par
	Fortunately, the empirical optimal solution has an explicit formulation, so we can directly compare it with the well-trained model. We empirically find that though the two models are close in each time step, the slight existing difference caused by optimization bias is critical for the diffusion model to generalize. This observation suggests that the neural network has the ``regularization'' property brought by the training stage  \citep{zhang2021understanding}. We propose another training objective to verify the conclusion to get the diffusion model. The empirical optima of the proposed objective is shown to fix the generalization problem of the original one. We compare the models trained by the proposed and original objectives. The empirical results indicate that the two models have similar outputs, so we conclude that the potential generalization problem of diffusion can be obviated during the training of neural networks. 
	
	\section{Related Work}
	\paragraph{Generalization of Generative Model.} The classical generalization theory in prediction measures the gap between the model's performance on training and test data \citep{duchi2016lecture,vapnik1999nature,yi2022characterization}. However, as the learned generative model does not take training data as input, the classical generalization theory does directly applied. To the best of our knowledge, \citep{arora2017generalization} explore the generalization of GAN, while their definition measures the gap between population distance and empirical distance of the target and generated distributions. However, this notation is inconsistent with the intuition that a generalizable model can generate data that does not exist in the training set. 
	\par
	We measure generalization by correlating the generated and training data. The criterion is consistent with the intuition of generalization of the generative model, as we claimed in Section \ref{sec:intro}. The idea also originates from the informatic-generalization bound \citep{xu2017information,yi2023breaking,bu2020tightening,lopez2018generalization}, which says the correlation decides the generalization of the prediction problem between the model and training set.        
	\paragraph{Denoising Diffusion Probabilistic Model.} The milestone work \citep{sohl2015deep} constructs a general formulation of the denoising diffusion probabilistic model, then specializes it by Gaussian and Binomial noises. By developing the Gaussian framework (diffusion model), \citep{song2020score,ho2020denoising} obtain remarkable high-quality generated data. Thus, for the diffusion model, the left is verifying its generalization property. Though \citep{somepalli2022diffusion,carlini2023extracting} shows there are some generated samples that are quite similar to training data which may threaten the privacy of the diffusion model, our results show that the diffusion model can obviate memorizing training data \citep{somepalli2022diffusion,carlini2023extracting}.        
	\par
	On the other hand, to get the diffusion model, we usually minimize the problems of noise prediction \citep{ho2020denoising,song2020score} or data prediction \citep{cao2022survey,gu2022vector}. We propose to minimize the ``previous points'' to get a diffusion model, and we prove the proposed objective can obviate the potential generalization problem of the diffusion model.     
	
	\section{Excess Risk of Generative Model}\label{sec:excess risk of generative model}
	In this section, we formally define the excess risk \citep{yi2022characterization} of the generative model, which evaluates the performance of it. Let training set $\bS=\{\bx^{i}_{0}\}_{i=1}^{n}$ be the $n$ i.i.d. samples from target distribution $P_{0}$ with bounded support $\cX$. The parameterized generative model $f_{\btheta_{\bS}}(\cdot)$ with $\btheta_{\bS}$ related to the training set $\bS$ transforms the variable $\bv$ to the generated data $\bz=f_{\btheta_{\bS}}(\bv)$ such that $\bz\sim Q_{\btheta_{\bS}}$, where the $\bv$ can be easily sampled e.g., Gaussian \citep{kingma2013auto,goodfellow2014generative}. 
	\par
	Intuitively, the ideal generative model is making $Q_{\btheta_{\bS}}$ close to the target distribution $P_{0}$, but $\bz\sim Q_{\btheta_{\bS}}$ is less related to training set $\bS$ so that it generalize. The latter obviates the model generates data via memorizing the training set. For example, taking $Q_{\btheta_{\bS}}$ as empirical distribution will generate data only from the training set. Though such $Q_{\btheta_{\bS}}$ can converge to target distribution \citep{wainwright2019}, it clearly can not generalize. The following is the former definition of excess risk. 
	\begin{definition}[Excess Risk]
		Let $\bz^{j}\sim Q_{\btheta_{\bS}}$ generated by model $f_{\btheta_{\bS}}$, then the excess risk of $f_{\btheta_{\bS}}$ is 
		\begin{equation}\label{eq:excess risk}
			\small
			d_{\cF}(Q_{\btheta_{\bS}}, P_{0}) = \sup_{g\in\cF}\left|\mE_{\bS}\left[\limsup_{m\rightarrow\infty}\frac{1}{m}\sum\limits_{j=1}^{m}g(\bz^{j}, \bS) - \mE_{\bx\sim P_{0}}[g(\bx, \bS)]\right]\right|,
		\end{equation}
		where $\cF = \{g(\bx, \bS): g(\bx, \bS)\in C(\cX, \cX^{n})\}$. 
	\end{definition}
	Our definition originates from the probabilistic distance named integral probability metric (IPM) which is defined as 
	\begin{equation}\label{eq:ipm}
		\small
		d_{\cF}(P, Q) = \sup_{f\in\cF}\left|\mE_{P}[X] - \mE_{Q}[X]\right|.
	\end{equation}
	Clearly, only if $Q_{\btheta_{\bS}}$ is close to $P_{0}$ for any $g(\bx, \bS)\in\cF$, we can infinitely sample $\bz$ and taking average to approximate $\mE_{P_{0}}[g(\bx, \bS)]$. The correlation between $\bz$ and $\bS$ is induced by making $g(\bx, \bS)$ take $\bS$ as input so that the correlation between $\bz$ and $\bS$ is involved in the excess risk. For example, the ideal model is making $\bz^{j}$ independent with $\bS$, if $Q_{\btheta_{\bS}} = P_{0}$, then $\limsup_{m\rightarrow\infty}\frac{1}{m}\sum_{j=1}^{m}g(\bz^{j}, \bS) \rightarrow \mE_{\bz\sim P_{0}}[g(\bz, \bS)]$, and the excess risk becomes zero. The following theorem which is proved in Appendix \ref{app:proofs in section excess risk of generative model} formulates the excess risk as an IPM.  
	\begin{restatable}{theorem}{probabilitydistance}\label{thm:excess risk formulation}
		If the generated data $\bz^{j}$ in \eqref{eq:excess risk} are conditional independent with each other, given the training set $\bS$, and $\cF$ has countable dense set under $L_{\infty}$ distance, then the excess risk \eqref{eq:excess risk} becomes
		\begin{equation}
			\small
			d_{\cF}(Q_{\btheta_{\bS}}, P_{0}) = \sup_{g\in\cF}\left|\mE_{\bS}\left[\mE_{\bz\sim Q_{\btheta_{\bS}}}[g(\bz, \bS)] - \mE_{\bx\sim P_{0}}[g(\bx, \bS)]\right]\right|. 
		\end{equation}
	\end{restatable}
	The conditional independence can be satisfied by many of generative models, e.g., GAN, diffusion model, VAE. Thus we explore the excess risk under such conditions in the sequel. At first glance, we can decompose it as 
	\begin{equation}\label{eq:decom}
		\small
		\begin{aligned}
			d_{\cF}(Q_{\btheta_{\bS}}, P_{0}) & \leq \sup_{g\in\cF}\left|\mE_{\bS, \bS^{\prime}}\left[\mE_{\bz\sim Q_{\btheta_{\bS}}}[g(\bz, \bS)] - \mE_{\bz\sim Q_{\btheta_{\bS}}}[g(\bz, \bS^{\prime})]\right]\right| \\
			& + \sup_{g\in\cF}\left|\mE_{\bS, \bS^{\prime}}\left[\mE_{\bz\sim Q_{\btheta_{\bS}}}[g(\bz, \bS^{\prime})] - \mE_{\bz\sim P_{0}}[g(\bz, \bS)]\right]\right| \\
			& \leq \underbrace{D_{\cF}(P_{\bz_{\btheta_{\bS}}\times \bS}, P_{\bz_{\btheta_{\bS}}}\times P_{\bS})}_{\rm generalization \ error} + \underbrace{D_{\cF}(Q_{\btheta_{\bS}}, P_{0})}_{\rm optimization \ error}, 
		\end{aligned}
	\end{equation}
	where $D_{\cF}(P, Q)$ is IPM defined in \eqref{eq:ipm}, and $\bS^{\prime}$ is another data set from $P_{0}$ independent with $\bS$. We explain the two terms in the above inequality. At first glance, the optimization error measures the distance between of generated distribution and the target one, which is the classical metric to evaluate the quality of generated data, e.g., JS-divergence \citep{goodfellow2014generative}, KL-divergence \citep{kingma2013auto,song2021maximum}, and Wasserstein distance \citep{arjovsky2017wasserstein}. On the other hand, the generalization error term measures the distance between union distribution $P_{\bz_{\btheta_{\bS}}\times \bS}$ and $P_{\bz_{\btheta_{\bS}}}\times P_{\bS}$. This is decided by the correlation between $\bz_{\btheta_{\bS}}$ and training set $\bS$, which intuitively represents the generalization ability of the generative model. A similar correlation has been well explored in informatic-generalization theory \citep{xu2017information,rodriguez2021tighter}. In their works, the generalization error of the prediction problem is decided by probabilistic distance with $\bz$ substituted by the learned parameters. Finally, we make several examples to illustrate our excess risk in Appendix \ref{app:examples}.    
	\par
	As the generalization error should be influenced by the number of samples \citep{vapnik1999nature}. To reduce such influence, we have the following proposition, in which we also link the generalization term to practical mutual information whose definition can be found in \citep{duchi2016lecture}.   
	\begin{restatable}{proposition}{generalization}\label{pro:generalization def}
		Suppose $g(\bz, \bS)\in\cF$ takes the form of $\frac{1}{n}\sum_{i=1}^{n}f(\bz, \bx^{i}_{0})$ such that $\mE_{Q_{\btheta_{\bS}\times P_{0}}}[\exp f(\bz, \bx)] < \infty$ and $|f(\bz, \bx)| \leq M$, then 
		\begin{equation}
			\small
			d_{\cF}(Q_{\btheta_{\bS}}, P_{0}) \leq \sqrt{\frac{M^{2}}{n}I(\bz_{\btheta_{\bS}}, \bS)} + d_{\cF_{P_{0}}}(Q_{\btheta_{\bS}}, P_{0}),
		\end{equation}
		where $\cF_{P_{0}} = \{\mE_{\bx\sim P_{0}}[f(\bz, \bx)]: |f(\bz, \bx)| \leq M; \mE_{Q_{\btheta_{\bS}\times P_{0}}}[\exp f(\bz, \bx)] < \infty\}$ and $d_{\cF_{P_{0}}}(Q_{\btheta_{\bS}}, P_{0}) \leq \max\{D_{KL}(P_{0}, Q_{\btheta_{\bS}}), D_{KL}(Q_{\btheta_{\bS}}, P_{0})\}$. 
	\end{restatable}
	The proof of this theorem is in Appendix \ref{app:proofs in section excess risk of generative model}. As can be seen, when we restrict the estimated term $g(\bz, \bS)$ as the average over the training set, the generalization can be related to the number of training samples, which is consistent with our common sense. Besides that, the generalization error is decided by the mutual information between generated data and the training set. 
	\section{Excess Risk of Diffusion Model}\label{sec:excess risk of diffusion model}
	As we defined the excess risk to evaluate the generative model in Section \ref{sec:excess risk of generative model}, we apply it to the diffusion model in the sequel. 
	\subsection{Revisiting Diffusion Model}
	As in \citep{ho2020denoising}, take $\bx_{0}\sim P_{0}$, and construct a forward process $\{\bx_{1}, \cdots, \bx_{T}\}$ such that $\bx_{t + 1}\mid \bx_{t}\sim\cN(\sqrt{1 - \beta_{t}}\bx_{t}, \beta_{t}\bI)$, with $\beta_{t} > 0$ is variance schedule. By simple computation, we get 
	\begin{equation}\label{eq:noise schedule}
		\small
		\bx_{t} = \sqrt{\bar{\alpha}_{t}}\bx_{0} + \sqrt{1 - \bar{\alpha}_{t}}\beps_{t},
	\end{equation}
	where $\beps_{t}$ is a standard Gaussian noise independent with $\bx_{0}$ and $\alpha_{t} = 1 - \beta_{t}$, $\bar{\alpha}_{t} = \prod_{1\leq s\leq_{t}}\alpha_{s}$.	As can be seen, by properly designing $\beta_{t}$, the forward process obtains $\bx_{T}$ that is close to a standard Gaussian distribution. Then to reversely generate $\bx_{0}$, we can consider a reversed Markov process $\bx_{t}$ such that $Q_{\btheta}(\bx_{t - 1}\mid \bx_{t}) = \cN(\bmu_{\btheta}(\bx_{t}, t), \Sigma_{\btheta}(\bx_{t}, t))$. Since $\bx_{T}\approx \cN(0, \bI)$, we can get $\bx_{t - 1}$ by iteratively sampling from $Q_{\btheta}(\bx_{t - 1}\mid \bx_{t})$, starting with a $\bx_{T}$ sampled from standard Gaussian. To get transition probability $Q_{\btheta}(\bx_{t - 1}\mid \bx_{t})$, consider the constructed variational bound of maximal likelihood loss
	\begin{equation}\label{eq:elbo}
		\small
		\begin{aligned}
			\mE_{\bx_{0}\sim P_{0}}\left[-\log{Q_{\btheta}(\bx_{0})}\right] & \leq \mE_{P}\left[-\log{\frac{Q_{\btheta}(\bx_{0:T})}{P(\bx_{1: T}\mid \bx_{0})}}\right] \\
			& = C + \mE\left[\sum_{t > 1}\underbrace{D_{KL}(P(\bx_{t - 1}\mid \bx_{t}, \bx_{0}) \parallel Q_{\btheta}(\bx_{t - 1} \mid \bx_{t}))}_{L_{t - 1}} \underbrace{-\log{Q_{\btheta}(\bx_{0}\mid \bx_{1})}}_{L_{0}}\right],
		\end{aligned}
	\end{equation}  
	where $C$ is a constant independent with $\btheta$. The update rule of $Q_{\btheta}(\bx_{t - 1}\mid \bx_{t})$ can be obtained via minimizing $L_{\rm vb}=\sum_{t=0}^{T - 1}L_{t}$. By Bayes's rule, we have $P(\bx_{t - 1}\mid \bx_{t}, \bx_{0})\sim\cN(\tilde{\bmu}_{t}(\bx_{t}, \bx_{0}), \tilde{\beta}_{t}\bI)$ with 
	\begin{equation}
		\small
		\tilde{\bmu}_{t}(\bx_{t}, \bx_{0}) = \frac{\sqrt{\bar{\alpha}_{t - 1}}\beta_{t}}{1 - \bar{\alpha}_{t}}\bx_{0} + \frac{\sqrt{\alpha_{t}}(1 - \bar{\alpha}_{t - 1})}{1 - \bar{\alpha}_{t}}\bx_{t}; \qquad \tilde{\beta_{t}} = \frac{1 - \bar{\alpha}_{t - 1}}{1 - \bar{\alpha}_{t}}\beta_{t},
	\end{equation} 
	Then we can explicitly get the optimal solution for each of $L_{t - 1}$ by selecting proper $\bmu_{\btheta}(\bx_{t}, t)$ and $\Sigma_{\btheta}(\bx, t)$. We have the following proposition proved in Appendix \ref{app:proofs in section excess risk of diffusion model} to characterize the transition probability kernel $Q_{\btheta}(\bx_{t - 1}\mid \bx_{t})$ for $t > 1$. On the other hand, as in \citep{song2020denoising}, the transition probability kernel of $Q_{\btheta}(\bx_{0}\mid \bx_{1})$ is usually set as the mean in (\ref{eq:optimal transition probability}).   
	\begin{restatable}{proposition}{argminkl}\label{pro:argmin}
		For $\bmu_{\btheta}(\bx_{t}, t)$ with enough functional capacity, then 
		\begin{equation}\label{eq:optimal transition probability}
			\small
			\argmin_{\bmu_{\btheta}(\bx_{t}, t)} L_{t - 1} = \tilde{\bmu}_{t}\left(\bx_{t}, \mE\left[\bx_{0}\mid \bx_{t}\right]\right); \qquad \argmin_{\Sigma_{\btheta}(\bx_{t}, t)} L_{t - 1} = \tilde{\beta}_{t}.  
		\end{equation}
	\end{restatable}
	In the widely used denoising diffusion probabilistic model (DDPM \citep{ho2020denoising}) the transition rule is 
	\begin{equation}\label{eq:ho reverse process}
		\small
		\bmu_{\btheta}(\bx_{t}, t) = \frac{1}{\sqrt{\alpha_{t}}}\left(\bx_{t} - \frac{\beta_{t}}{\sqrt{1 - \bar{\alpha}_{t}}}\beps_{\btheta}^{*}(\bx_{t}, t)\right); \qquad \Sigma_{\btheta}(\bx_{t}, t) = \tilde{\beta}_{t},
	\end{equation}
	where $\epsilon_{\btheta}^{*}(\bx_{t}, t)$ is a parameterized model such that $\epsilon_{\btheta}^{*} = \inf_{\beps_{\btheta}}\mE_{\bx_{t}, \beps_{t}}[\|\beps_{\btheta}(\bx_{t}, t) - \beps_{t}\|^{2}]$. According to the optimality of conditional expectation under minimizing expected square loss \citep{banerjee2005optimality}, we know that the ideal
	\begin{equation}\label{eq:eps}
		\small
		\beps_{\btheta}^{*}(\bx_{t}, t) = \mE[\beps_{t}\mid \bx_{t}] = \mE\left[\frac{1}{\sqrt{1 - \bar{\alpha}_{t}}}\bx_{t} - \frac{\sqrt{\bar{\alpha}_{t}}}{\sqrt{1 - \bar{\alpha}_{t}}}\bx_{0}\mid \bx_{t}\right]. 
	\end{equation}
	By plugging this into \eqref{eq:ho reverse process}, we get $\bmu_{\btheta}(\bx_{t}, t)$ is exactly the proposed optimal $\tilde{\bmu}_{t}\left(\bx_{t}, \mE\left[\bx_{0}\mid \bx_{t}\right]\right)$. Thus, the rationale of standard DDPM is matching $P(\bx_{t - 1}\mid \bx_{t}, \bx_{0})$ by substituting $\bx_{0}$ with conditional expectation $\mE[\bx_{0}\mid \bx_{t}]$. On the other hand, Proposition \ref{pro:argmin} indicates that such substitution is optimal in terms of minimizing variational bound $L_{\rm vb}$. 
	\subsection{Excess Risk of Diffusion Model}
	We have pointed out the optimal transition rule of the diffusion model above. Next, we verify the excess risk of the diffusion model under such a rule to generate data. In practice, to approximate the model $\beps_{\btheta}^{*}$ after \eqref{eq:ho reverse process}, we minimize the following empirical counterpart of noise prediction problem $\inf_{\beps_{\btheta}}\mE_{\bx_{t}, \beps_{t}}\left[\left\|\beps_{\btheta}(\bx_{t}, t) - \beps_{t}\right\|^{2}\right]$.
	\begin{equation}\label{eq:empirical objective}
		\small
		\inf_{\beps_{\btheta}}\frac{1}{n}\sum\limits_{i=1}^{n}\mE_{\beps_{t}}\left[\left\|\beps_{\btheta}(\sqrt{\bar{\alpha}_{t}}\bx^{i}_{0} + \sqrt{1 - \bar{\alpha}_{t}}\beps_{t}, t) - \beps_{t}\right\|^{2}\right].
	\end{equation}
	The following two theorems explore the excess risk of the optima of \eqref{eq:empirical objective}. 
	\begin{restatable}{theorem}{generalizationofode}\label{pro:generalization of ddpm}
		Suppose the model $\beps_{\btheta}(\cdot, \cdot)$ has enough functional capacity,  let $\beps_{\btheta_{\bS}}^{*}(\bx, t)$ be any optima of \eqref{eq:empirical objective}, then 
		\begin{equation}\label{eq:empirical optimal eps}
			\small
			\begin{aligned}
				\beps_{\btheta_{\bS}}^{*}(\bx, t) 
				= \frac{\bx}{\sqrt{1 - \bar{\alpha}_{t}}} - \left(\frac{\sqrt{\bar{\alpha}_{t}}}{\sqrt{1 - \bar{\alpha}_{t}}}\right)\sum_{i=1}^{n}\frac{\exp\left(-\frac{\|\bx - \sqrt{\bar{\alpha}_{t}}\bx^{i}_{0}\|^{2}}{2(1 - \bar{\alpha}_{t})}\right)\bx^{i}_{0}}{\sum_{i=1}^{n}\exp\left(-\frac{\|\bx - \sqrt{\bar{\alpha}_{t}}\bx^{i}_{0}\|^{2}}{2(1 - \bar{\alpha}_{t})}\right)}.
			\end{aligned}
		\end{equation} 
		Then if the transition rule of DDPM satisfies \eqref{eq:ho reverse process} as in \citep{ho2020denoising}, we have  
		\begin{equation}
			\small
			I(\bx_{0}, \bS) \leq \frac{(1 - \beta_{1})R^{2}}{2\beta_{1}^{2}} + \sum_{t = 2}^{T}\frac{\bar{\alpha}_{t}R^{2}}{2(1 - \bar{\alpha}_{t - 1})^{2}}
		\end{equation}
		where $\bx_{0}$ is generated by the model, then the generalization error in Proposition \ref{pro:generalization def} is upper bounded.  
	\end{restatable}
	The proof of this theorem is in Appendix \ref{app:empirical optima}. The theorem indicates that the empirical optima of the noise prediction problem can have guaranteed generalization error when the R.H.S. of the above inequality is small. This happens when $1 / \beta_{1}$ is not extremely large, which requires when constructing noisy data $\{\bx_{t}\}$ the first $\bx_{1}$ should be pretty noisy according to \eqref{eq:forward process}.  
	\par
	Next, we use the following theorem to indicate that such empirical optima also converge to the optimal model $\mE[\beps_{t}\mid \bx_{t}]$ as discussed in \eqref{eq:eps}. Thus its ability to generate high-quality data is also guaranteed, as $\mE[\beps_{t}\mid \bx_{t}]$ minimizes $L_{\rm vb}$, which is an upper bound of KL-divergence between generated distribution and target one (measures optimization error). The following theorem is proved in \ref{app:empirical minima}. 
	\par
	\begin{restatable}{theorem}{convergenceoriginal}\label{thm:sample complexity original}
		Let $\beps_{\btheta_{\bS}}^{*}(\cdot, \cdot)$ be the model defined in \eqref{eq:empirical optimal eps}, then for any $t$, and $\bx_{t}$ with bounded norm, we have $\beps_{\btheta_{\bS}}^{*}(\bx_{t}, t)\overset{P}\longrightarrow \mE[\beps_{t}\mid \bx_{t}]$.  
	\end{restatable}
	\par
	Combining Theorem \ref{pro:generalization of ddpm} and \ref{thm:sample complexity original}, we conclude that for a training set with a sufficiently large number, the DDPM can have guaranteed excess risk (under small $\beta_{1}$) so that generating high-quality data with a small dependence on the training set. However, the transition rule of DDPM is low efficient owing to a large $T$ in practice, e.g., 1000 in \citep{ho2020denoising}. Because getting every $\bx_{t}$ during generation requires taking a forward propagation of learned model $\beps_{\btheta_{\bS}}$, which takes plenty of computational costs. Researchers have proposed a deterministic reverse process (e.g., DDIM \citep{song2020denoising}), which can generate high-quality data with fewer steps during its reverse process.
	\par
	Unfortunately, as can be seen in \eqref{eq:empirical optimal eps}, the empirical optima $\beps_{\btheta_{\bS}}^{*}(\bx, t)$ takes the form of a linear combination of difference between the $\bx$ and training set. Thus, any deterministic reverse process to generate $\bx_{t}$ will make the generated data highly dependent on the training set, then poor generalization. The formal results is stated in the following proposition.   
	\begin{restatable}{proposition}{genofdeterministic}\label{pro:gen of deterministic}
		If the transition rule of the diffusion model takes the form of $\bx_{t - 1} = f(\beps_{\btheta_{\bS}}^{*}, \bx_{t}, t)$ for some deterministic $f$. Then the generalization error of the diffusion model is infinity. 
	\end{restatable}
	To clarify the poor generalization, we take DDIM \citep{song2020denoising} as an example. The $\bx_{t - 1}$ in DDIM is generating via a linear combination of $\bx_{t}$ and $\beps_{\btheta_{\bS}}^{*}(\bx_{t}, t)$, which results in the generated $\bx_{0}$ must be a linear combination of training set. Clearly, we do not want such generated data as they only depend on the training set. Compared with DDPM, the guaranteed generalization of DDPM \eqref{eq:ho reverse process} originates the injected noise during generating process, which decreases the dependence between $\bx_{t}$ and the training set. 
	\par
	The rationale for causing such a problem is the optimal model $\beps_{\btheta}^{*}$ in \eqref{eq:eps} involves conditional expectation $\mE[\bx_{0}\mid \bx_{t}]$, and we require training set to estimate it. Since  
	\begin{equation}\label{eq:estimating conditional expectation}
		\small
		\mE[\bx_{0}\mid \bx_{t}] = \int_{\cX}\bx_{0}P(\bx_{0}\mid \bx_{t})d\bx_{0} = \int_{\cX}\bx_{0}\frac{P(\bx_{t}\mid \bx_{0})}{P(\bx_{t})}P(\bx_{0})d\bx_{0} \approx \frac{1}{n}\sum\limits_{i=1}^{n}\bx^{i}_{0}\frac{P(\bx_{t}\mid \bx^{i}_{0})}{\hat{P}(\bx_{t})},
	\end{equation}
	with $\hat{P}(\bx_{t})$ as a proper estimation to $P(\bx_{t})$, the estimator can be easily highly related to the training set. This can be verified by combining \eqref{eq:eps} and \eqref{eq:empirical optimal eps}. 
	\section{The Optimization Bias Improves Generalization}\label{sec:optimization bias}
	\begin{figure*}[t!]\centering
		\subfloat[Averaged distance $\|\bx_{t} - \bx_{t}^{*}\|^{2}$ per dimension. \label{fig:distance}]{
			\includegraphics[width=0.45\textwidth]{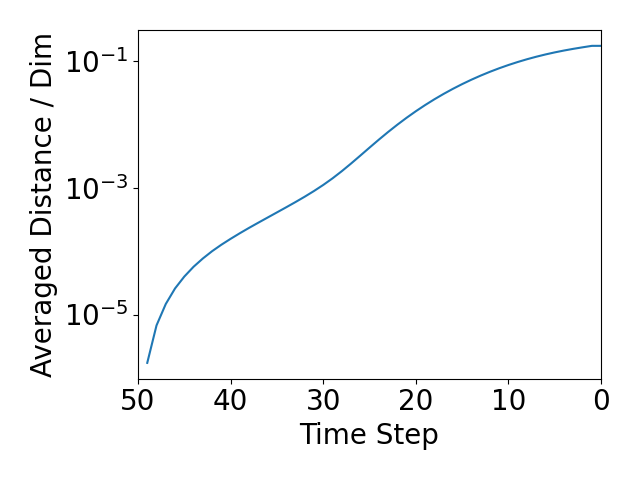}}
		\hspace{0.1in}
		\subfloat[Generated $\bx_{t}$ and $\bx^{*}_{t}$ \label{fig:genereated data}]{
			\includegraphics[width=0.5\textwidth]{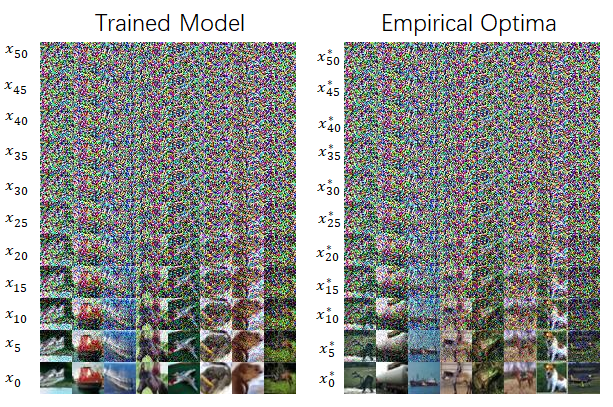}}
		\caption{The first figure is the averaged distance $\|\bx_{t} - \bx_{t}^{*}\|$ per dimension (3$\times$32$\times$32) over 50k samples of generated \texttt{CIFAR10}. The second figure randomly samples a batch of $\bx_{t}$ and $\bx_{t}^{*}$ with the same $\bx_{T}=\bx_{T}^{*}$ and $T=50$.
		}
		\label{fig:comparsion}
	\end{figure*}
	As we have claimed above, the empirical optima have a potential generalization problem. Unfortunately, this problem may be transferred to the sufficiently trained model as it approximates the empirical optima. Thus in this section, we explore whether the sufficiently trained model has generalization problem. 
	\subsection{The Optimization Bias Regularizes Diffusion Model}
	Fortunately, we have the explicit formulation of empirical optima as in \eqref{eq:empirical optimal eps}. Thus we can directly compare it with a sufficiently trained model. As we have claimed, when generating data with a deterministic reverse process, the empirical optima may generate data highly related training set. We explore it and verify whether this happens to the trained model.  
	\par
	Following the pipeline in \citep{ho2020denoising}, we train a deep neural network, i.e., Unet \citep{ronneberger2015u} on an image data set \texttt{CIFAR10} \citep{krizhevsky2009learning} to verify the difference. We use 50-steps deterministic reverse process DDIM \footnote{The 50-step DDIM used here follows the one in \citep{song2020denoising}. After training a diffusion model build on $\bx_{0:1000}$, then using $\bx_{\lceil ci\rceil}$ with $i=1,\cdot, 50$ and $c=25$, as the new $\{\bx_{t}\}$.} as in \citep{song2020denoising} such that 
	\begin{equation}
		\small
		\bx_{t + 1} = \sqrt{\bar{\alpha}_{t - 1}}\left(\frac{\bx_{t} - \sqrt{1 - \bar{\alpha}_{t}}\beps_{\btheta(\bx_{t}, t)}}{\sqrt{\bar{\alpha}_{t}}}\right) + \sqrt{1 - \bar{\alpha}_{t - 1}}\beps_{\btheta}(\bx_{t}, t)
	\end{equation}
	to generate data. Let $\bx_{t}$ and $\bx_{t}^{*}$ respectively be the data generated by our trained model and the empirical optima. That means substituting $\beps_{\btheta}$ in the above equation with the trained model and empirical optima. We randomly sample 50K standard Gaussian and feed them into our trained and empirical optimal diffusion models. To check the difference, we summarize the averaged $l_{2}$-distance $\|\bx_{t} - \bx_{t}^{*}\|^{2}$ per dimension over the 50K iterates in Figure \ref{fig:distance}. We also randomly sample some iterates $\{\bx_{t}\}$ and $\{\bx_{t}^{*}\}$ to visualize the difference in Figure \ref{fig:genereated data}.   
	\par
	As can be seen, the distance between $\bx_{t}$ and $\bx_{t}^{*}$ increased with time step $t$. This is a natural result, as there is a gap between the trained model and the empirical optima owing to the bias brought by the optimization process. Then the difference will cumulatively increase, resulting in different generated data as shown in Figure \ref{fig:genereated data}. Thus we can conclude that the optimization bias regularizes the trained model to perfectly fit empirical optima, which instead potentially obviates the generalization problem.    
	\subsection{The Optimization Bias Helps Extrapolating}\label{sec:The Optimization Bias Helps Extrapolating}
	\begin{figure*}[t!]\centering
		\subfloat[The generted data starting from noisy data in test set of \texttt{CIFAR10}. \label{fig:train}]{
			\includegraphics[width=0.48\textwidth]{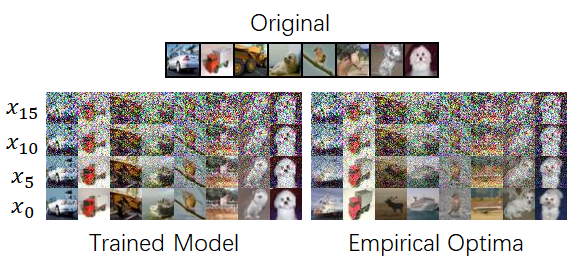}}
		\hspace{0.1in}
		\subfloat[Generted data started from noisy data in the training set.\label{fig:test}]{
			\includegraphics[width=0.48\textwidth]{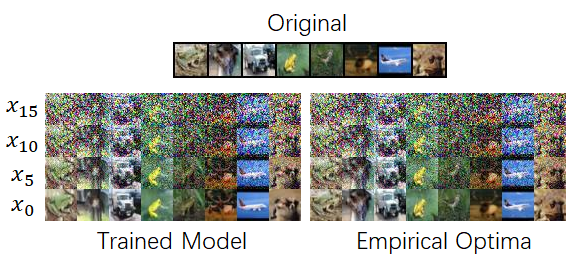}}
		\caption{The data in the two top figures are $\bx_{0}$ respectively from test and training sets of \texttt{CIFAR10}. The bottom are the data generated by the trained model (left) and empirical optima (right).
		}
		\label{fig:generated data with initial}
	\end{figure*}
	Though the optimization process implicitly regularizes the trained model to generate data with the one generated by empirical optima. We should examine whether the two models will generate data that existed in the training set. Unfortunately, we observe that nearly all data generated by empirical optima exist in the training set, which also verifies our conclusion that the model has a generalization problem. On the other hand, for the trained model, we also compare the nearest data in the training set with its generated data to examine its generalization. Fortunately, we found that nearly all data does not appear in the training set. Thus, the optimization bias guarantees the extrapolation ability of the model. This phenomenon is shown in Appendix \ref{app:accu}.  
	\par
	To further verify the extrapolation ability of the trained diffusion model, we explicitly show that it can generate data that does not exist in the training set. Instead of starting the reverse process from $\bx_{50}$, we use $\bx_{15}$ ($\bx_{15} = 0.6678\bx_{0} + 0.7743\beps_{t}$) as an initial point so that we can check the generating process more clear. The $\bx_{0}$ we choose to $\bx_{15}$ are from the test set, and a batch of generated data is in Figure \ref{fig:test}. As can be seen, the trained diffusion model nearly recovers the original data from the noisy ones. Nevertheless, for empirical optima, though starting from the same points, it can not recover the original unseen test data. By the way, decreasing the total sampling steps does change the result since the accumulated bias depends on the distance between $\bx_{T}$ and $\bx_{0}$ (see \eqref{eq:forward process}) which does influence by the number of sampling steps.  
	\par
	It has been observed in \citep{carlini2023extracting} that the trained diffusion model occasionally generates data close to the one in the training set. Even though we think this does not threaten the extrapolation ability of the diffusion model. We conduct another experiment to explain such a phenomenon. Similar to the generating process in Figure \ref{fig:test}, we generated data by the diffusion model and empirical optima, starting from $\bx_{15}$ but with $\bx_{0}$ drawn from the training set.
	\par
	A batch of generated data is in Figure \ref{fig:train}. As can be seen, both the trained diffusion model and empirical optima can recover the original data from the noisy one. This explains why the diffusion model generates data in the training set. The generating happens if the reverse process moves to $\bx_{t}$ around noisy data close to the one potentially constructed by the training set (like the $\bx_{15}$) when $t$ is close to zero. The repeating generation is because the gap between generated data caused by the optimization bias of the trained diffusion model does not accumulate enough to regularize the training process. However, we think such repeating could hardly happen as $\bx_{t}$ locates in high-dimensional space, so noisy data generated by the training set is sparse in its support \citep{wainwright2019high}. Oppositely, with enough accumulated bias, we observe that such a phenomenon does not happen when taking $t=50$ as in Figure \ref{fig:genereated data} even though with $\bx_{50}$ generated by the training set. We verify it in Appendix \ref{app:accu}.    
	\par
	Finally, we point out that this empirically observed phenomenon also holds for reverse process DDPM \citep{ho2020denoising}. As we have shown in Theorem \ref{pro:generalization of ddpm}, the generalization problem is resolved when $\beta_{1}$ is large, which does not hold for the one of DDPM ($\beta_{1}=0.0001$). 
	\section{Estimating Previous Status Improves Generalization}\label{sec:estimating forward}
	As we have pointed out in \eqref{eq:estimating conditional expectation}, the potentially broken generalization property of the diffusion model originates from estimating $\mE[\beps_{t}\mid \bx_{t}]$ (equivalent to estimating $\mE[\bx_{0}\mid \bx_{t}]$), which may lead the generated data highly related to the training set. Though this phenomenon can be mitigated by the optimization bias. We propose another training objective to get a diffusion model and generate data. Unlike the one of \eqref{eq:empirical objective}, the empirical optima of our proposed training objective mitigate the potential generalization problem.
	\par
	Actually, we can rewrite the Proposition \ref{pro:argmin} such that 
	\begin{proposition}
		For $\bmu_{\btheta}(\bx_{t}, t)$ with enough functional capacity, then 
		\begin{equation}\label{eq:optimal transition probability xtm1}
			\small
			\argmin_{\bmu_{\btheta}(\bx_{t}, t)} L_{t - 1} = \mE[\bx_{t - 1}\mid \bx_{t}]; \qquad \argmin_{\Sigma_{\btheta}(\bx_{t}, t)} L_{t - 1} = \tilde{\beta}_{t}.  
		\end{equation}
	\end{proposition} 
	As can be seen, in contrast to the transition probability rule in \eqref{eq:optimal transition probability}, the new rule does not involve $\mE[\bx_{0}\mid \bx_{t}]$, so that it potentially obviates the generalization problem. Naturally, we may consider solving  $\inf_{\bx_{\btheta}}\mE[\|\bx_{t - 1} - \bx_{\btheta}(\bx_{t}, t)\|^{2}]$ ($\bx_{\btheta}(\cdot, \cdot)$ is the parameterized diffusion model) to get $\mE[\bx_{t - 1} \mid \bx_{t}]$, as it is the solution of the minimization problem. However, practically, we found that the minimizing $\mE[\|\bx_{t - 1} - \bx_{\btheta}(\bx_{t}, t)\|^{2}]$ is unstable. We speculate this is due to the $\bx_{t}$ and $\bx_{t - 1}$ are so close which makes $\bx_{\btheta}(\bx_{t}, t)$ rapidly converges to identity map of $\bx_{t}$ for each $t$. 
	\par
	Owing to the aforementioned training problem, we consider another method to estimate the $\mE[\bx_{t - 1} \mid \bx_{t}]$. Suppose that  
	\begin{equation}\label{eq:forward process}
		\small
		\bx_{t} = \sqrt{\frac{\bar{\alpha}_{t}}{\bar{\alpha}_{s}}}\bx_{s} + \sqrt{1 - \frac{\bar{\alpha}_{t}}{\bar{\alpha}_{s}}}\bxi_{t,s},
	\end{equation}
	and $\bar{\alpha}_{t} / \bar{\alpha}_{s} = r_{t, s}$, $\bxi_{t, s}\sim\cN(0, \bI)$. Then 
	\begin{equation}
		\small
		\mE\left[\bx_{t - 1}\mid \bx_{t}\right] = \frac{1}{\sqrt{r_{t, t - 1}}}\bx_{t} - \sqrt{\frac{1 - r_{t, t - 1}}{r_{t, t - 1}}}\mE\left[\bxi_{t, t - 1} \mid \bx_{t}\right].
	\end{equation}
	Thus estimating $\mE\left[\bx_{t - 1}\mid \bx_{t}\right]$ is equivalent to estimating $\mE\left[\bxi_{t, t - 1} \mid \bx_{t}\right]$. To get it, we have the following lemma, which is known as Tweedie’s formula \citep{efron2011tweedie}. 
	\begin{restatable}{lemma}{scoreequivalent}\label{lem:equivalence}
		For and $s < t$, we have $\frac{\mE[\bxi_{t, s}\mid \bx_{t}]}{\sqrt{1 - r_{t, s}}} = \frac{\mE\left[\bxi_{t, t - 1} \mid \bx_{t}\right]}{\sqrt{1 - r_{t, t - 1}}} = -\nabla_{\bx_{t}}\log{P_{t}}(\bx_{t})$.
	\end{restatable}
	From the above lemma, we know that estimating $\mE\left[\bxi_{t, t - 1} \mid \bx_{t}\right]$ is equivalent to estimate $\mE[\bxi_{t, s}\mid \bx_{t}]$ for any $0 \geq s < t$, but the difference between $\bx_{t}$ and $\bx_{s}$ can be large when $s$ is far away from $t$. We empirically find that a large gap benefits the optimization process. Thus our training objective becomes 
	\begin{equation}\label{eq:our objective population}
		\small
		\inf_{\bxi_{\btheta}}\sum_{t=1}^{T}\mE_{s}\left[\mE_{\bx_{s}, \bxi_{t, s}}\left[\left\|\frac{\bxi_{t, s}}{\sqrt{1 - r_{t, s}}} - \bxi_{\btheta}(\sqrt{r_{t, s}}\bx_{s} + \sqrt{1 - r_{t, s}}\bxi_{t, s}, t)\right\|^{2}\right]\right],
	\end{equation}
	where $s$ follows any distribution, e.g., uniform in $\{0,\cdots, T - 1\}$, and $\bxi_{\btheta}$ is the final parameterized diffusion model. This can be done as for any specific $t$ and $s$, the problem of minimizing $\mE_{\bx_{s}, \bxi_{t, s}}\left[\|\bxi_{t, s} / \sqrt{1 - r_{t, s}} - \bxi_{\btheta}(\sqrt{r_{t, s}}\bx_{s} + \sqrt{1 - r_{t, s}}\bxi_{t, s}, t)\|^{2}\right]$ has common global optima $\mE[\bxi_{t, t - 1}\mid \bx_{t}] / \sqrt{1 - r_{t, t - 1}}$ due to Lemma \ref{lem:equivalence} and \citep{banerjee2005optimality}. Practically, let us consider the empirical counterpart of the above problem such that 
	\begin{equation}\label{eq:our objective}
		\small
		\inf_{\bxi_{\btheta}}\sum_{t=1}^{T}\mE_{s}\left[\frac{1}{n}\sum_{i=1}^{n}\mE_{\bxi_{t, s}}\left[\left\|\frac{\bxi_{t, s}}{\sqrt{1 - r_{t, s}}} - \bxi_{\btheta}(\sqrt{r_{t, s}}\bx_{s}^{i} + \sqrt{1 - r_{t, s}}\bxi_{t, s}, t)\right\|^{2}\right]\right]
	\end{equation}  
	The $\{\bx_{s}^{i}\}$ is generated through training set that follows the distribution of $\bx_{t}$. The objective is actually equivalent to the (7) in reverse-SDE \citep{song2020score} but substituting $\bx_{0}$ with $\bx_{s}$ as discussed in \ref{app:proofs of section estimating forward}. The following proposition gives the empirical optimal of \eqref{eq:our objective}. 
	\begin{restatable}{proposition}{ourempiricalsolution}
		Suppose the model $\bxi_{\btheta}(\cdot, \cdot)$ has enough functional capacity the optimal solution of \eqref{eq:our objective} is 
		\begin{equation}\label{eq:our empirical solution}
			\small
			\bxi_{\btheta_{\bS}}^{*}(\bx, t) = \sum_{i=1}^{n}\frac{\mE_{s}\left[\left(\frac{1}{2\pi(1 - r_{t, s})}\right)^{\frac{d}{2}}\exp\left(-\frac{\|\bx - \sqrt{r_{t, s}}\bx_{s}^{i}\|^{2}}{2(1 - r_{t, s})}\right)\left(\frac{\bx - \sqrt{r_{t, s}}\bx_{s}^{i}}{1 - r_{t,s}}\right)\right]}{\sum_{i=1}^{n}\mE_{s}\left[\left(\frac{1}{2\pi(1 - r_{t, s})}\right)^{\frac{d}{2}}\exp\left(-\frac{\|\bx - \sqrt{r_{t, s}}\bx_{s}^{i}\|^{2}}{2(1 - r_{t, s})}\right)\right]}.
		\end{equation}
	\end{restatable}
	As can be seen, in contrast to \eqref{eq:empirical optimal eps}, the optimal solution $\bxi_{\btheta_{\bS}}^{*}(\bx, t)$ does not highly relate to the training set. It involves $\{\bx_{s}\}$ for series of $s$ (depending on the distribution of $s$), and these $\{\bx_{s}\}$ are noisy data generated by training set. Thus, despite the optimization bias discussed in Section \ref{sec:optimization bias}, updating with \eqref{eq:our empirical solution} does not cause the potential generalization problem. The proof of this theorem is in Appendix \ref{app:proofs of section estimating forward}. 
	\par
	Similar to the Theorem \ref{thm:sample complexity original}, the proposed $\bxi_{\btheta_{\bS}}^{*}(\cdot, \cdot)$ also converges to its approximation target $\mE[\bxi_{t - 1} \mid \bx_{t}]$, so that it has small optimization error when $n$ is large enough. The result is illustrated in the following theorem, which is proved in \ref{app:proofs of section estimating forward}. 
	\begin{restatable}{theorem}{convergenceofourobjective}
		Let $\bxi_{\btheta_{\bS}}(\cdot, \cdot)$ be the model defined in \eqref{eq:our empirical solution}, then for any $t$ and $\bx_{t}$ with bounded norm, we have $\bxi_{\btheta_{\bS}}^{*}(\bx_{t}, t)\overset{P}{\longrightarrow}\mE[\bxi_{t, t - 1}\mid \bx_{t}] / \sqrt{1 - r_{t, t - 1}}$. 
	\end{restatable}
	\section{Experiments}\label{sec:exp}
	\begin{figure*}[t!]\centering
		\subfloat[Averaged distance per dimension on \texttt{CIFAR10}. \label{fig:cifar10}]{
			\includegraphics[width=0.3\textwidth]{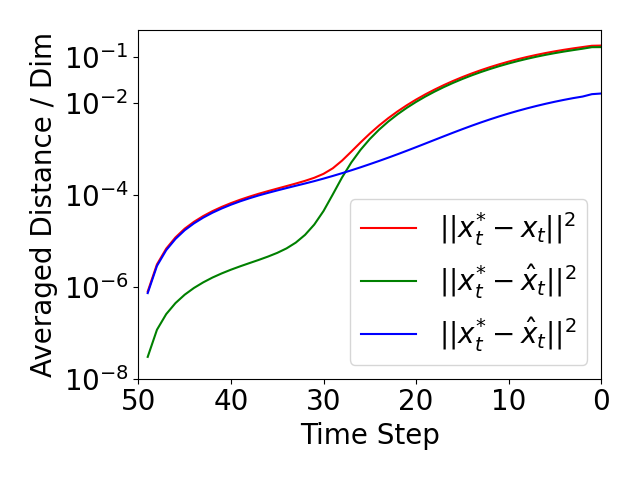}}
		\hspace{0.1in}
		\subfloat[Averaged distance per dimension on \texttt{CelebA} \label{fig:celeba}]{
			\includegraphics[width=0.3\textwidth]{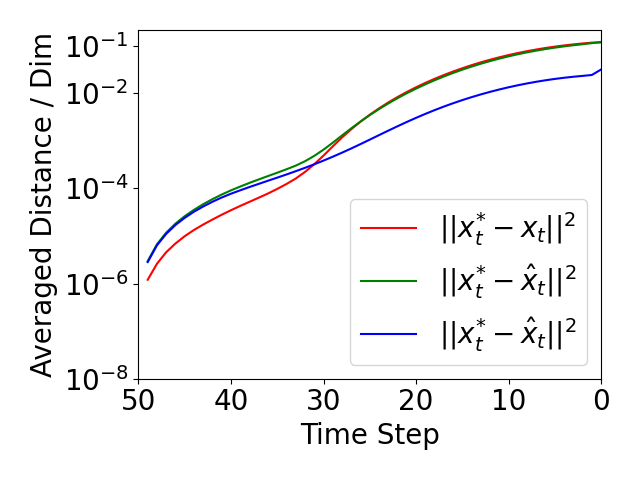}}
		\subfloat[Comparasion of $\bx_{0}, \hat{\bx}_{0}, \bx_{0}^{*}$. \label{fig:comp}]{
			\includegraphics[width=0.3\textwidth]{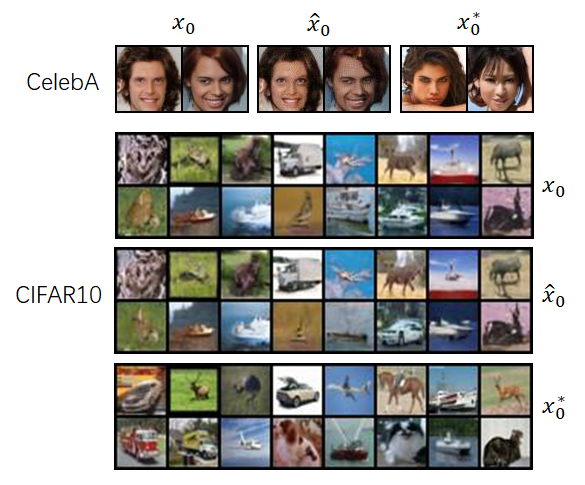}}
		\caption{The comparisons of $\bx_{t}, \hat{\bx}_{t}, \bx_{t}^{*}$, where they are respectively generated by diffusion models trained by \eqref{eq:empirical objective}, \eqref{eq:our objective}, and the empirical optima \eqref{eq:empirical optimal eps}.
		}
		\label{fig:distance total}
	\end{figure*}
	In Section \ref{sec:optimization bias}, we empirically verify that although the empirical optimal diffusion model has a generalization problem, i.e., generating data from the training set. The optimization bias regularizes the trained diffusion model and enables it to generalize. In this section, we further verify the generalization capability of the diffusion model. We have shown in Section \ref{sec:estimating forward} proposed training objective \eqref{eq:our objective} can obviate the potential generalization problem. Thus in this section, we empirically verify the difference between diffusion models trained by \eqref{eq:empirical objective}, \eqref{eq:our objective}, and the empirical optima \eqref{eq:empirical optimal eps}.
	\par
	\paragraph{Setup.} Our experimental settings are similar to the ones in Section \eqref{sec:optimization bias}. That is, taking the reverse process as 50 steps DDIM \citep{song2020denoising} to generate data. The diffusion models trained \eqref{eq:empirical objective} and \eqref{eq:our objective} are Unets \citep{ronneberger2015u} with size depending on the dataset as in \citep{nichol2021improved}. To get $s$ when training diffusion under our objective \eqref{eq:our objective}, for $k=t - s$, we first uniformly sample $k$ from $1, \cdots, T - 1$, then uniformly sample a $s$ from $0,\cdots, T - k$. In addition, the sampled $\bx_{s}$ during the training stage is generated by the training set according to \eqref{eq:forward process}. The other experimental settings follow the ones in \citep{ho2020denoising}.
	\paragraph{Datasets.} Our experiments are conducted on image datasets \texttt{CIFAR10} \citep{krizhevsky2009learning}, \texttt{CelebA} \citep{liu2015deep}, which are all benchmark datasets with size $32\times32$ and $64\times64$. 
	\paragraph{Main Results.} Similar to Section \ref{sec:optimization bias}. Let $\bx_{t}^{*}$, $\bx_{t}$ and $\tilde{\bx}_{t}$ respectively be the iterates generated by empirical optima, diffusion models trained by \eqref{eq:empirical objective} and \eqref{eq:our objective}. For each model, we generate 50k series of iterates to compare the average difference (per dimension) between them. The comparisons are summarized in Figure \ref{fig:distance total}. As can be seen, compared with $\bx_{t}^{*}$, the iterates generated by trained diffusion models are pretty similar. This illustrates that though they are trained by different objectives, but the optimization bias pushes them towards a similar model with generalization ability. As the model trained by \eqref{eq:our objective} does not have a potential generalization problem, the similarity between $\bx_{t}$ and $\hat{\bx}_{t}$ indicates the generalization ability of diffusion model trained by \eqref{eq:our objective}. 
	\par
	Some samples generated by the three models are in Figure \ref{fig:comp}. As can be seen, the $\bx_{0}$ and $\hat{\bx}_{0}$ are visually close to each other, while $\hat{\bx}_{0}$ are noisy compared with $\bx_{0}$. In fact, $\hat{\bx}_{0}$ exhibits higher FID score (lower is better)\citep{heusel2017gans} than $\bx_{0}$ (evaluated as in \citep{ho2020denoising}), that is 11.30 v.s. 3.17 on \texttt{CIFAR10} and 108.73 v.s. 8.91 on \texttt{CelebA}. We speculate this is because $\hat{\bx}_{t}$ are more noisy, which improves generalization but increases the optimization error. This illustrates that there is a trade-off between the two errors. Thus when training the diffusion model, we should consider balancing them. 
	\section{Conclusion}
	In this paper, we first formally define the excess risk of the generative model to evaluate it. The excess risk can be decomposed into optimization and generalization errors, which relate to the quality of generated data and the model's exploration ability, respectively. We mainly focus on exploring the generalization of the diffusion model. We verify that though the empirical optimal diffusion model has poor generalization, the optimization bias brought by the training stage of the diffusion model enables it to generate high quality, but meanwhile preserving the generalization ability.                
	\newpage
	\bibliography{reference}
	\bibliographystyle{apalike}
	\appendix
	\include{appendix} 
\end{document}

%% file: appendix.tex
\section{Proofs in Section \ref{sec:excess risk of generative model}}\label{app:proofs in section excess risk of generative model}
\probabilitydistance*
\begin{proof}
	As given the training set $\bS$, the generated data $\bz^{j}$ are conditional independent with each other. Then for any $g$, and a realization of training set $\bS_{0}$, we have that 
	\begin{equation}
		\small
		\limsup_{m\rightarrow\infty}\frac{1}{m}g(\bz^{j}, \bS_{0}) = \mE_{\bz\sim Q_{\btheta_{\bS_{0}}}}[g(\bz, \bS_{0})], \qquad a.s., 
	\end{equation}
	where $a.s.$ means almost surely. Due to $\cX$ is bounded, it has countable dense sets. Then for any dense set $\cX^{n}_{0}$ of $\cX^{n}$, we have the above equation holds for any $\bS_{0}\in\cX^{n}_{0}$ almost surely. Then for any $\bS\in\cX^{n}$, due to the continuity of $g$ w.r.t. $\bS$, and $\cX^{n}_{0}$ is a dense subset of $\cX^{n}$, we have 
	\begin{equation}
		\small
		\begin{aligned}
			\limsup_{m\rightarrow\infty}\frac{1}{m}g(\bz^{j}, \bS) - \mE_{\bz\sim Q_{\btheta_{\bS}}}[g(\bz, \bS)] &= \limsup_{m\rightarrow\infty}\frac{1}{m}g(\bz^{j}, \bS) - \limsup_{m\rightarrow\infty}\frac{1}{m}g(\bz^{j}, \bS^{\epsilon}) \\
			& + \mE_{\bz\sim Q_{\btheta_{\bS^{\epsilon}}}}[g(\bz, \bS^{\epsilon})] - \mE_{\bz\sim Q_{\btheta_{\bS}}}[g(\bz, \bS)] \\
			& \leq 2\cO(\epsilon), \qquad a.s.
		\end{aligned}
	\end{equation}
	where $\bS^{\epsilon}\in\cX_{0}^{n}$ such that $\|\bS^{\epsilon} - \bS\|\leq \epsilon$. Then due to the arbitrary of $\epsilon$, we get that 
	\begin{equation}
		\small
		\mE_{\bS}\left[\limsup_{m\rightarrow\infty}\frac{1}{m}g(\bz^{j}, \bS)\right] = \mE_{\bS}\mE_{\bz\sim Q_{\btheta_{\bS}}}[g(\bz, \bS)], \qquad a.s.
	\end{equation}
	holds for any fixed $g$. For any countable dense set $\cF_{0}$ of $\cF$, we have 
	\begin{equation}
		\small
		d_{\cF_{0}}(Q_{\btheta_{\bS}}, P_{0}) = \sup_{g\in\cF_{0}}\left|\mE_{\bS}\left[\mE_{\bz\sim Q_{\btheta_{\bS}}}[g(\bz, \bS)] - \mE_{\bx\sim P_{0}}[g(\bx, \bS)]\right]\right|. \qquad a.s. 
	\end{equation} 
	Then for any $\epsilon > 0$, there must exists dense set $\cF_{0}$ such that 
	\begin{equation}\label{eq:eps inequality}
		\small
		\begin{aligned}
			& \left|d_{\cF}(Q_{\btheta_{\bS}}, P_{0}) - \sup_{g\in\cF}\left|\mE_{\bS}\left[\mE_{\bz\sim Q_{\btheta_{\bS}}}[g(\bz, \bS)] - \mE_{\bx\sim P_{0}}[g(\bx, \bS)]\right]\right|\right| \\
			& \leq \left|d_{\cF}(Q_{\btheta_{\bS}}, P_{0}) - d_{\cF_{0}}(Q_{\btheta_{\bS}}, P_{0})\right| \\
			& + \left|d_{\cF}(Q_{\btheta_{\bS}}, P_{0}) - \sup_{g\in\cF}\left|\mE_{\bS}\left[\mE_{\bz\sim Q_{\btheta_{\bS}}}[g(\bz, \bS)] - \mE_{\bx\sim P_{0}}[g(\bx, \bS)]\right]\right|\right| \\
			& + \left|\sup_{g\in\cF}\left|\mE_{\bS}\left[\mE_{\bz\sim Q_{\btheta_{\bS}}}[g(\bz, \bS)] - \mE_{\bx\sim P_{0}}[g(\bx, \bS)]\right]\right| - \sup_{g\in\cF_{0}}\left|\mE_{\bS}\left[\mE_{\bz\sim Q_{\btheta_{\bS}}}[g(\bz, \bS)] - \mE_{\bx\sim P_{0}}[g(\bx, \bS)]\right]\right|\right| \\
			& \leq 3\epsilon. 
		\end{aligned}
	\end{equation}
	Let us define event 
	\begin{equation} 
		\small
		A = \left\{d_{\cF}(Q_{\btheta_{\bS}}, P_{0})\neq \sup_{g\in\cF}\left|\mE_{\bS}\left[\mE_{\bz\sim Q_{\btheta_{\bS}}}[g(\bz, \bS)] - \mE_{\bx\sim P_{0}}[g(\bx, \bS)]\right]\right|\right\}.
	\end{equation} 
	Then $A = \bigcup_{\epsilon > 0}A^{\epsilon} = \bigcup_{j=1}^{\infty}A^{\frac{1}{j}}$, with 
	\begin{equation}
		\small
		A^{\epsilon} = \left\{\left|d_{\cF}(Q_{\btheta_{\bS}}, P_{0}) - \sup_{g\in\cF}\left|\mE_{\bS}\left[\mE_{\bz\sim Q_{\btheta_{\bS}}}[g(\bz, \bS)] - \mE_{\bx\sim P_{0}}[g(\bx, \bS)]\right]\right|\right|\geq \epsilon \right\}.
	\end{equation}
	Due to the denseness of $\cF_{0}$, we can get 
	\begin{equation}
		\small
		\bbP(A) = \bbP\left(\bigcup_{\epsilon > 0}A^{\epsilon}\right) = \bbP\left(\bigcup_{j=1}^{\infty}A^{\frac{1}{j}}\right) \leq \sum_{j=1}^{\infty}\bbP(A^{\frac{1}{j}}) = 0,
	\end{equation}
	where the last equality is due to \eqref{eq:eps inequality}. Thus we prove our result. 
\end{proof}
\generalization*
\begin{proof}
	Let us check the generalization error first. Due to the formulation of $g(\bz, \bS)$, for any $\lambda > 0$
	\begin{equation}
		\small
		\begin{aligned}
			& \lambda\sup_{g\in\cF}\mE_{\bS, \bS^{\prime}}\left[\mE_{\bz\sim Q_{\btheta_{\bS}}}[g(\bz, \bS)] - \mE_{\bz\sim Q_{\btheta_{\bS}}}[g(\bz, \bS^{\prime})]\right] = \lambda\mE_{\bS, \bz\sim Q_{\btheta_{\bS}}}\left[\frac{1}{n}\sum_{i=1}^{n}\left(f(\bz, \bx^{i}_{0}) - \mE_{\bS^{\prime}}[f(\bz, \bx^{i^\prime}_{0})]\right)\right] \\
			& \leq D_{KL}(P_{\bz_{\btheta_{\bS}}\times \bS}, P_{\bz_{\btheta_{\bS}}}\times P_{\bS}) + \log{\mE_{\bS^{\prime}, \bz\sim Q_{\btheta_{\bS}}}\left[\exp\left(\frac{\lambda}{n}\sum_{i=1}^{n}\left(f(\bz, \bx^{i}_{0}) - \mE_{\bS^{\prime}}[f(\bz, \bx^{i^\prime}_{0})]\right)\right)\right]} \\
			& \overset{a}{\leq} I(\bz_{\btheta_{\bS}}, \bS) + \frac{\lambda^{2}M^{2}}{2n},
		\end{aligned}
	\end{equation}
	where inequality $a$ is from the sub-Gaussian property \citep{duchi2016lecture}. By taking infimum of $\lambda$, and similarly applying the result to $\lambda < 0$, we prove an upper bound to the generalization error that is  
	\begin{equation}
		\small 
		D_{\cF}(P_{\bz_{\btheta_{\bS}}\times \bS}, P_{\bz_{\btheta_{\bS}}}\times P_{\bS}) \leq \sqrt{\frac{M^{2}I(\bz_{\btheta_{\bS}}, \bS)}{n}}. 
	\end{equation}
	On the other hand, for the optimization error, then 
	\begin{equation}
		\small
		\begin{aligned}
			& \sup_{g\in\cF}\mE_{\bS, \bS^{\prime}}\left|\left[\mE_{\bz\sim Q_{\btheta_{\bS}}}[g(\bz, \bS^{\prime})] - \mE_{\bz\sim P_{0}}[g(\bz, \bS)]\right]\right| \\
			& = \sup_{f}\left|\mE_{\bx\sim P_{0}}\left[\left(\mE_{\bz\sim Q_{\btheta_{\bS}}}\left[f(\bz, \bx)\right] - \mE_{\bz\sim P_{0}}\left[f(\bz, \bx)\right]\right)\right] \right| \\
			& = d_{\cF_{P_{0}}}(Q_{\btheta_{\bS}}, P_{0}),
		\end{aligned}
	\end{equation}
	where $\cF_{P_{0}} = \{\mE_{\bx\sim P_{0}}[f(\bz, \bx)]: |f(\bz, \bx)| \leq M; \mE_{Q_{\btheta_{\bS}\times P_{0}}}[\exp f(\bz, \bx)] < \infty\}$. Then due to the Donsker–Varadhan representation \citep{duchi2016lecture}, we have $d_{\cF_{P_{0}}}(Q_{\btheta_{\bS}}, P_{0}) \leq \max\{D_{KL}(P_{0}, Q_{\btheta_{\bS}}), D_{KL}(Q_{\btheta_{\bS}}, P_{0})\}$, which proves our theorem. 
\end{proof}
\section{Some Examples of Excess Risk}\label{app:examples}
To make our excess risk more practical, we use the following example to illustrate the effectiveness. 
\begin{example}
	The excess risk of empirical distribution $Q_{\btheta_{\bS}} = P_{n}$. 
\end{example}
According to Theorem \ref{thm:excess risk formulation}, the excess risk of empirical distribution is 
\begin{equation}
	\small
	\begin{aligned}
		d_{\cF}(P_{n}, P_{0}) & = \left|\sup_{g\in\cF}\mE_{\bS}\left[\frac{1}{n}\sum\limits_{i=1}^{n}g(\bx^{i}_{0}, \bS) - \mE_{\bx\sim P_{0}}[g(\bx, \bS)]\right]\right| \\
		& \geq \left|\mE_{\bS}\left[\frac{1}{n^{2}}\sum\limits_{i = 1}^{n}\sum\limits_{j=1}^{n}\frac{1}{\|\bx^{i}_{0} - \bx^{j}\|^{2}}\right] - \frac{1}{n}\sum\limits_{i=1}^{n}\mE_{\bx\sim P_{0}}\frac{1}{\|\bx - \bx^{i}_{0}\|^{2}}\right| \\
		& = \infty.
	\end{aligned} 
\end{equation}
As can be seen, when we involve the generalization into excess risk, the empirical distribution has poor performance, which is consistent with our intuition. However, under the original metric, the empirical distribution can have great performance with sufficiently large $n$. This is because the existing evaluation \citep{kingma2013auto,song2020score,arora2017generalization} is usually probabilistic distance or divergence between generated and target distributions, e.g., Wasserstein distance or KL divergence. However, the classical results of empirical process \citep{van1996weak} indicate that the empirical distribution can converge $P_{0}$ under these metrics. This contradicts to our intuition that memorizing training data is a bad behavior for the generative model. 
\par
Next, let us present an example to exactly compute the excess risk of generative model.  
\begin{example}
	Let i.i.d. training set $\bS$ such that $\bx_{i}\sim\cN(\bmu, \bI)$. Our goal is using $\bS$ to estimate $\bmu$ by $\hat{\bmu}$ and generate data $\bz\sim\cN(\hat{\bmu}, \bI)$. Now let us check the generalization and optimization error of generated data $\bz$.
\end{example}  
The classical way to get $\hat{\bmu}$ is minimizing the square loss $\frac{1}{n}\sum_{i=1}^{n}\|\hat{\bmu} - \bx_{i}\|^{2}$, which is obtained by $\hat{\bmu} = \frac{1}{n}\sum_{i=1}^{n}\bx_{i}$. Thus, $\bz\sim\cN(\hat{\bmu}, \bI)$. We consider the function class as $\cF = \{f:|f| \leq M\}$ for some $M > 0$. As in the proof of proposition \ref{pro:generalization def}, the $D_{\cF}(\cdot, \cdot)$ defined in \eqref{eq:decom} can be upper bounded by $\max\{D_{KL}(Q_{\hat{\bmu}} || P_{0}), D_{KL}(P_{0} || Q_{\hat{\bmu}})\}$ by applying Jensen's inequality and its definition. Thus, the optimization error of $Q_{\hat{\bmu}}$ can be explicitly computed due to KL-divergence between two Gaussian distributions \citep{duchi2016lecture} such that 
\begin{equation}
	\small
		D_{\cF}(P_{0}, Q_{\hat{\bmu}}) \leq D_{KL}(Q_{\hat{\bmu}} || P_{0}) = D_{KL}(P_{0} || Q_{\hat{\bmu}}) = \mE[\|\hat{\bmu} - \bmu\|^{2}] = \frac{d}{n}.
\end{equation}
On the other hand, for the generalization error, similar to the proof of Proposition \ref{pro:generalization def}, for some standard Gaussian distribution $\bxi$, we have 
\begin{equation}
	\small
	\begin{aligned}
		D_{\cF}(P_{{\bz}\times \bS}, Q_{\hat{\bmu}}\times P_{\bS}) & \leq I(\bz; \bS) \\
		&= \sum\limits_{i=1}^{n}I(\bz; \bx_{i}\mid \bx_{1: i - 1}) \\
		& = \sum\limits_{i=1}^{n}I\left(\frac{1}{n}\sum_{i=1}^{n}\bx_{i} + \bxi; \bx_{i}\mid \bx_{1: i - 1}\right) \\
		& = \sum\limits_{i=1}^{n}\left(H\left(\frac{1}{n}\sum_{i=1}^{n}\bx_{i} + \bxi\mid \bx_{1: i - 1}\right) - H\left(\frac{1}{n}\sum_{i=1}^{n}\bx_{i} + \bxi\mid \bx_{1: i}\right)\right) \\
		& = H\left(\frac{1}{n}\sum_{i=1}^{n}\bx_{i} + \bxi\right) + H\left(\bxi\right) \\
		& = \frac{d}{2}\log{\left(1 + \frac{1}{n}\right)}, 
	\end{aligned}
\end{equation}
where the last equality is due to the entropy of Gaussian distribution \citep{duchi2016lecture}. Thus we respectively characterize the upper bounds of generalization and optimization errors.

\section{Proofs in Section \ref{sec:excess risk of diffusion model}}\label{app:proofs in section excess risk of diffusion model}
We prove a general result such that our Proposition \ref{pro:argmin} is a corollary of it. We first present the definition of exponential family distributions, which is adopted from \citep{duchi2016lecture} 
\begin{definition}[Exponential Family Distributions]
	The exponential family associated with the function $\phi(\cdot)$ is defined as the set of distributions with densities $Q_{\btheta}$, where 
	\begin{equation}
		\small
			Q_{\btheta}(\bx) = \exp\left(\left\langle\btheta, \phi(\bx)\right\rangle - A(\btheta)\right),
	\end{equation}
	and the function $A(\btheta)$ is the log-partition-function defined as 
	\begin{equation}
		A(\btheta) = \log{\int_{\cX}\exp\left(\langle\btheta, \phi(\bx)\rangle \right)d\bx}
	\end{equation} 
\end{definition}
Before proving Proposition \ref{pro:argmin}, we need the following lemma. 
\begin{lemma}\label{lem:kl divergence for exp}
	For densities functions $P(\cdot)$ and $Q_{\btheta}(\cdot)$, if $Q_{\btheta}(\cdot)$ is an exponential family variable, then 
	\begin{equation}\label{eq:minimize kl}
		\small
			\btheta^{*} = \argmin_{\btheta}D_{KL}(P\parallel Q_{\btheta}) = \nabla A^{-1}(\mE_{p}[\phi(\bx)]),
	\end{equation}
	and $\mE_{Q_{\btheta^{*}}}[\phi(\bx)] = \mE_{P}[\phi(\bx)]$, where $\nabla A^{-1}(\btheta)$ is the inverse of $\nabla A(\btheta)$, due to the convexity of $A(\btheta)$
\end{lemma}
\begin{proof}
	From the definition 
	\begin{equation}
		\small
		\begin{aligned}
			D_{KL}(P\parallel Q_{\btheta}) &= \int_{\cX}P(\bx)\log{P(\bx)}d\bx - \int_{\cX}P(\bx)\log{Q_{\btheta}(\bx)}d\bx \\
			& = \int_{\cX}p(\bx)\log{P(\bx)}d\bx - \int_{\cX}P(\bx)\left(\langle\btheta, \phi(\bx) - A(\btheta)\right)d\bx.
		\end{aligned}
	\end{equation}
	Then minimizing $D_{KL}(P\parallel Q_{\btheta})$ is equivalent to maximizing $- \int_{\cX}P(\bx)\left(\langle\btheta, \phi(\bx) - A(\btheta)\right)d\bx$. According to Proposition 14.4 in \citep{duchi2016lecture}, $A(\btheta)$ is a convex function w.r.t. $\btheta$, then let $Q_{\btheta^{*}}$ solves \eqref{eq:minimize kl}, we must have $\nabla A(\btheta^{*}) = \mE_{\bx\sim P}[\phi(\bx)]$. On the other hand, we have 
	\begin{equation}
		\small
			\nabla A(\btheta) = \frac{\int_{\cX}\exp\left(\langle\btheta, \phi(\bx)\rangle\right)\phi(\bx)d\bx}{\int_{\cX}\exp\left(\langle\btheta, \phi(\bx)\rangle \right)d\bx} = \mE_{\bx\sim Q_{\btheta}}[\phi(\bx)],
	\end{equation}
	which verifies the second conclusion. 
\end{proof}
\argminkl*
\begin{proof}
	The normal distribution is exponential family with the form $Q_{\btheta, \Sigma}(\bx) \propto \exp(\langle\btheta, \bx\rangle + 1/2\langle\bx\bx^{\top}, \Sigma\rangle)$, where $\Sigma$ is the covariance matrix of $Q_{\btheta, \Sigma}(\bx)$ and $\btheta$ is $\Sigma^{-1}\bmu$ with $\bmu$ is the mean of $Q_{\btheta, \Sigma}(\bx)$. Then the result is a corollary of Lemma \ref{lem:kl divergence for exp} due to the linearity of $\tilde{\bmu}(\bx_{0}, \bx_{t})$ w.r.t. $\bx_{0}$.  
\end{proof}
\subsection{The Empirical Optima of Noise Prediction}\label{app:empirical optima}
Next we prove the Theorem \ref{pro:generalization of ddpm}.
\generalizationofode*
\begin{proof}
	Let
	\begin{equation}
		\small
		\begin{aligned}
			J(\beps_{\btheta}) & = \frac{1}{n}\sum\limits_{i=1}^{n}\mE_{\beps_{t}}\left[\left\|\beps_{\btheta}(\sqrt{\bar{\alpha}_{t}}\bx^{i}_{0} + \sqrt{1 - \bar{\alpha}_{t}}\beps_{t}, t) - \beps_{t}\right\|^{2}\right] \\
			& = \frac{1}{n}\sum\limits_{i=1}^{n}\int_{\bbR^{d}}\left\|\beps_{\btheta}(\sqrt{\bar{\alpha}_{t}}\bx^{i}_{0} + \sqrt{1 - \bar{\alpha}_{t}}\beps_{t}, t) - \beps_{t}\right\|^{2}\left(\frac{1}{2\pi}\right)^{\frac{d}{2}}\exp\left(-\frac{\|\beps_{t}\|^{2}}{2}\right)d\beps_{t} \\
			& = \int_{\bbR^{d}}\frac{1}{n}\sum\limits_{i=1}^{n}\left\|\beps_{\btheta}(\bx, t) - \frac{\bx - \sqrt{\bar{\alpha}_{t}}\bx^{i}_{0}}{\sqrt{1 - \bar{\alpha}_{t}}}\right\|^{2}\left(\frac{1}{2\pi}\right)^{\frac{d}{2}}\exp\left(-\frac{\left\|\bx - \sqrt{\bar{\alpha}_{t}}\bx^{i}_{0}\right\|^{2}}{2(1 - \bar{\alpha}_{t})}\right)d\bx
		\end{aligned}
	\end{equation}
	For any given $\bx$, the optimization problem of minimizing $\beps_{\btheta}$ in the integral is a strongly convex problem w.r.t. $\beps_{\btheta}$. Thus it has single global minimum which can be obtained taking gradient to it such that   
	\begin{equation}
		\small
		\begin{aligned}
			0 &= \nabla_{\btheta}\frac{1}{n}\sum\limits_{i=1}^{n}\left\|\beps_{\btheta}(\bx, t) - \frac{\bx - \sqrt{\bar{\alpha}_{t}}\bx^{i}_{0}}{\sqrt{1 - \bar{\alpha}_{t}}}\right\|^{2}\left(\frac{1}{2\pi}\right)^{\frac{d}{2}}\exp\left(-\frac{\left\|\bx - \sqrt{\bar{\alpha}_{t}}\bx^{i}_{0}\right\|^{2}}{2(1 - \bar{\alpha}_{t})}\right)\\
			&= \frac{2}{n}\sum\limits_{i=1}^{n}\left(\beps_{\btheta}(\bx, t) - \frac{\bx - \sqrt{\bar{\alpha}_{t}}\bx^{i}_{0}}{\sqrt{1 - \bar{\alpha}_{t}}}\right)\left(\frac{1}{2\pi}\right)^{\frac{d}{2}}\exp\left(-\frac{\left\|\bx - \sqrt{\bar{\alpha}_{t}}\bx^{i}_{0}\right\|^{2}}{2(1 - \bar{\alpha}_{t})}\right), 
		\end{aligned}
	\end{equation}
	which shows 
	 \begin{equation}
	 	\small
	 	\begin{aligned}
	 		\beps_{\btheta_{\bS}}^{*}(\bx, t) & = \sum_{i=1}^{n}\frac{\exp\left(-\frac{\|\bx - \sqrt{\bar{\alpha}_{t}}\bx^{i}_{0}\|^{2}}{2(1 - \bar{\alpha}_{t})}\right)}{\sum_{i=1}^{n}\exp\left(-\frac{\|\bx - \sqrt{\bar{\alpha}_{t}}\bx^{i}_{0}\|^{2}}{2(1 - \bar{\alpha}_{t})}\right)}\left(\frac{\bx - \sqrt{\bar{\alpha}_{t}}\bx^{i}_{0}}{\sqrt{1 - \bar{\alpha}_{t}}}\right) \\
	 		& = \frac{\bx}{\sqrt{1 - \bar{\alpha}_{t}}} - \left(\frac{\sqrt{\bar{\alpha}_{t}}}{\sqrt{1 - \bar{\alpha}_{t}}}\right)\sum_{i=1}^{n}\frac{\exp\left(-\frac{\|\bx - \sqrt{\bar{\alpha}_{t}}\bx^{i}_{0}\|^{2}}{2(1 - \bar{\alpha}_{t})}\right)\bx^{i}_{0}}{\sum_{i=1}^{n}\exp\left(-\frac{\|\bx - \sqrt{\bar{\alpha}_{t}}\bx^{i}_{0}\|^{2}}{2(1 - \bar{\alpha}_{t})}\right)}
	 	\end{aligned}
	 \end{equation} 
 	\par
 	Next, we prove the claim of generalization, due to the Proposition \ref{pro:generalization def}, we should control the mutual information $I(\bx_{0}, \bS)$, where $\bx_{0}$ is obtained via 
 	\begin{equation}
 		\small
 		\bx_{t - 1} = \frac{1}{\sqrt{\alpha_{t}}}\left(\bx_{t} - \frac{\beta_{t}}{\sqrt{1 - \bar{\alpha}_{t}}}\beps_{\btheta}^{*}(\bx_{t}, t)\right) + \tilde{\beta}_{t}\bxi_{t},
 	\end{equation}
 	where $\bxi_{t}$ is a standard Gaussian that is independent of $\bx_{t}$ and $\bS$. Then by Data processing inequality \citep{xu2017information},
 	\begin{equation}
 		\small
 		\begin{aligned}
 			I(\bx_{0}; \bS) \leq I(\bx_{0:T}; \bS) = I(\bx_{0}; \bS \mid \bx_{1:T}) + I(\bx_{1}; \bS\mid \bx_{2:T}) + \cdots + I(\bx_{T}; \bS). 
 		\end{aligned}
 	\end{equation}
 	Then for any $1 \leq t \leq T$,
 	\begin{equation}
 		\small
 		\begin{aligned}
 			I(\bx_{t - 1}; \bS\mid \bx_{t: T}) = H(\bx_{t - 1}\mid \bx_{t: T}) - H(\bx_{t - 1}\mid \bS, \bx_{t: T}) = H(\bx_{t - 1}\mid \bx_{t}) - H(\bx_{t - 1}\mid \bS, \bx_{t}),
 		\end{aligned}
 	\end{equation}
 	where the last equality is due to the Markovian property of $\bx_{t - 1}$. Next we compute the two terms in the last equality. First, due to the definition of $\beps_{\btheta}^{*}$
 	\begin{equation}
 		\small
 		\begin{aligned}
 			H(\bx_{t - 1}\mid \bx_{t}) & = H\left(\bx_{t - 1} - \frac{1}{\sqrt{\alpha_{t}}}\bx_{t}\mid \bx_{t}\right) \\
 			& = H\left(\tilde{\beta}_{t}\bxi_{t} + \left(\beta_{t}\frac{\sqrt{\bar{\alpha}_{t}}}{1 - \bar{\alpha}_{t}}\right)\sum_{i=1}^{n}\frac{\exp\left(-\frac{\|\bx - \sqrt{\bar{\alpha}_{t}\bx^{i}_{0}}\|^{2}}{2(1 - \bar{\alpha}_{t})}\right)\bx^{i}_{0}}{\sum_{i=1}^{n}\exp\left(-\frac{\|\bx - \sqrt{\bar{\alpha}_{t}\bx^{i}_{0}}\|^{2}}{2(1 - \bar{\alpha}_{t})}\right)} \mid \bx_{t}\right). 
 		\end{aligned}
 	\end{equation}
 	Then since $\bx_{t}$ and $\bxi_{t}$ are independent we have 
 	\begin{equation}\label{eq:upper bound of l2}
 		\small
 		\begin{aligned}
 			& \mE\left[\left\|\tilde{\beta}_{t}\bxi_{t} + \left(\frac{\beta_{t}\sqrt{\bar{\alpha}_{t}}}{1 - \bar{\alpha}_{t}}\right)\sum_{i=1}^{n}\frac{\exp\left(-\frac{\|\bx - \sqrt{\bar{\alpha}_{t}\bx^{i}_{0}}\|^{2}}{2(1 - \bar{\alpha}_{t})}\right)\bx^{i}_{0}}{\sum_{i=1}^{n}\exp\left(-\frac{\|\bx - \sqrt{\bar{\alpha}_{t}\bx^{i}_{0}}\|^{2}}{2(1 - \bar{\alpha}_{t})}\right)}\right\|^{2}\right] \\
 			& = \tilde{\beta}_{t}^{2}d + \frac{\beta_{t}^{2}\bar{\alpha}_{t}}{(1 - \bar{\alpha}_{t})^{2}} \mE\left[\left\|\sum_{i=1}^{n}\frac{\exp\left(-\frac{\|\bx - \sqrt{\bar{\alpha}_{t}\bx^{i}_{0}}\|^{2}}{2(1 - \bar{\alpha}_{t})}\right)\bx^{i}_{0}}{\sum_{i=1}^{n}\exp\left(-\frac{\|\bx - \sqrt{\bar{\alpha}_{t}\bx^{i}_{0}}\|^{2}}{2(1 - \bar{\alpha}_{t})}\right)}\right\|^{2}\right] \\
 			& \leq \tilde{\beta}_{t}^{2}d + \frac{\beta_{t}^{2}\bar{\alpha}_{t}}{(1 - \bar{\alpha}_{t})^{2}}R^{2},
 		\end{aligned}
 	\end{equation}
 	where $R$ is the radius of data support $\cX$. Due to Theorem 14.7 in \citep{duchi2016lecture}, that among all random variables $X$ with $\mE[\|X\|^{2} \leq C]$, the Gaussian distribution $\cN(0, \sqrt{C / d}I_{d})$ has the largest entropy such that 
 	\begin{equation}
 		\small
 		H(\cN(0, \sqrt{C / d}I_{d})) = \frac{d}{2}\log{\left(\frac{2\pi eC}{d}\right)}.
 	\end{equation}
 	Combining this result with \eqref{eq:upper bound of l2}, we get 
 	\begin{equation}
 		\small
 		H(\bx_{t - 1}\mid \bx_{t}) \leq \frac{d}{2}\log{\left(2\pi e\left(\tilde{\beta}_{t}^{2} + \frac{\beta_{t}^{2}\bar{\alpha}_{t}}{d(1 - \bar{\alpha}_{t})^{2}}R^{2}\right)\right)}.
 	\end{equation}
 	On the other hand, due to the definition of $\beps_{\btheta}^{*}$, 
 	\begin{equation}
 		\small
 		H(\bx_{t - 1}\mid \bS, \bx_{t}) = H(\tilde{\beta}_{t}\bxi_{t}) = \frac{d}{2}\log(2\pi e\tilde{\beta}_{t}^{2}),
 	\end{equation}
 	which implies 
 	\begin{equation}
 		\small
 		I(\bx_{t - 1}; \bS\mid \bx_{t:T}) \leq \frac{d}{2}\log{\left(1 + \frac{\beta_{t}^{2}\bar{\alpha}_{t}}{d\tilde{\beta}_{t}^{2}(1 - \bar{\alpha}_{t})^{2}}R^{2}\right)} \leq \frac{\beta_{t}^{2}\bar{\alpha}_{t}R^{2}}{2\tilde{\beta}_{t}^{2}(1 - \bar{\alpha}_{t})^{2}} = \frac{\bar{\alpha}_{t}R^{2}}{2(1 - \bar{\alpha}_{t - 1})^{2}}. 
 	\end{equation}
 	We should point out that when $t=1$, the upper bounded in the above becomes $\frac{(1 - \beta_{1})R^{2}}{2\beta_{1}^{2}}$. Then we prove our result.
	\end{proof}
	\begin{remark}
		As we have shown in main text, the ideal  $\beps_{\btheta}$ is \eqref{eq:eps}, and the empirical $\beps_{\btheta_{\bS}}$ in \eqref{eq:empirical optimal eps} is approximating \eqref{eq:eps} as in \eqref{eq:estimating conditional expectation}. This conclusion can be easily verified due to $P(\bx_{t}\mid \bx_{0})\sim \cN(\sqrt{\bar{\alpha}_{t}}\bx_{0}, \sqrt{1 - \bar{\alpha}_{t}}\bI)$, by viewing the numerator of \eqref{eq:empirical optimal eps} as an unbiased estimator to $P(\bx_{t})$, which is $\hat{P}(\bx_{t})$ in \eqref{eq:estimating conditional expectation}. 
	\end{remark}
	In the last of this subsection, we use the following proposition to show the generalization problem of empirical when the diffusion model takes deterministic update rule e.g., DDIM \citep{song2020denoising}, DPM-Solver \citep{lu2022dpm}. 
	\genofdeterministic*
	\begin{proof}
		As can be seen, the $\beps_{\btheta_{\bS}}^{*}(\bx)$ is a linear combination of the difference between $\bx$ and training set $\bS$. Thus according to the transition rule $\bx_{t - 1} = f(\beps_{\btheta_{\bS}}^{*}, \bx_{t}, t)$, we know the generated data $\bx_{0}$ only depends on $\bS$ and $\bx_{T}$. Due to the linear formulation of $\beps_{\btheta_{\bS}}^{*}(\bx, t)$, there exists $\bx_{0} = F(\bx_{T}, \bS)$ with $F$ does not degenerated w.r.t. $\bS$. Thus 
		\begin{equation}
			\small
			I(\bx_{0}, \bS) = I(F(\bx_{T}, \bS); \bS) = I(F(\bx_{T}, \bS); \bS \mid \bx_{T}) + I(\bx_{T}; \bS) = I(F(\bx_{T}, \bS); \bS \mid \bx_{T}) = \infty,
		\end{equation}
		which verifies our conclusion. 
	\end{proof}
	The proposition indicates that though the deterministic update rule of diffusion model has improved sampling efficiency compared with the stochastic one \citep{song2020denoising,lu2022dpm}, but it potentially face the challenge of generalization.  
\subsection{Convergence of Empirical Minima}\label{app:empirical minima}
In this section, we prove the convergence result of empirical minima \eqref{eq:empirical optimal eps}. Before proving Theorem \ref{thm:sample complexity original}, we give some notations and present some useful lemmas. Let us define 
\begin{equation}
	\small
	\begin{aligned}
		K_{t}(\bx_{t}, \bx_{0}) &= \exp\left(-\frac{\|\bx - \sqrt{\bar{\alpha}_{t}}\bx_{0}\|^{2}}{2(1 - \bar{\alpha}_{t})}\right); \\
		f_{\bx_{0}}(\bx_{t}) &= \left(\frac{1}{2\pi(1 - \bar{\alpha}_{t})}\right)^{\frac{d}{2}}K_{t}(\bx_{t}, \bx_{0}); \\
		f_{\bS}(\bx_{t}) &= \frac{1}{n}\sum\limits_{i=1}^{n}f_{\bx_{0}^{i}}(\bx_{t}). 
	\end{aligned}
\end{equation}
\begin{lemma}\label{lem:continuity}
	The function $f_{\bS}(\bx_{t})$ and $P_{t}(\bx_{t})$ is $\left(\frac{1}{2\pi(1 - \bar{\alpha}_{t})}\right)^{\frac{d}{2}}\sqrt{\frac{1}{e(1 - \bar{\alpha}_{t})}}$-Lipschitz continuous and their gradients are all $\left(\frac{1}{2\pi(1 - \bar{\alpha}_{t})}\right)^{\frac{d}{2}}\left(\frac{2 + e}{e(1 - \bar{\alpha_{t}})}\right)$-Lipschitz continuous. 
\end{lemma}
\begin{proof}
	Due to the definition of $K_{t}(\bx_{t}, \bx_{0})$, we have 
	\begin{equation}
		\small
			\|\nabla_{\bx_{t}}K_{t}(\bx_{t}, \bx_{0})\| = -\exp\left(-\frac{\|\bx_{t} - \sqrt{\bar{\alpha}_{t}}\bx_{0}\|^{2}}{2(1 - \bar{\alpha}_{t})}\right)\left\|\frac{\bx_{t} - \sqrt{\bar{\alpha}_{t}}\bx_{0}}{1 - \bar{\alpha}_{t}}\right\| \leq \sqrt{\frac{1}{e(1 - \bar{\alpha}_{t})}},
	\end{equation}
	where we use the inequality $axe^{-\frac{ax^{2}}{2}} \leq \sqrt{a / e}$, then the Lipschitz continuity of $f_{\bS}(\bx_{t})$ and $P_{t}(\bx_{t})$ are directly obtained since $P_{t}(\bx_{t}) = \int_{\bbR^{d}}\left(\frac{1}{2\pi(1 - \bar{\alpha}_{t})}\right)^{\frac{d}{2}}P_{0}(\bx_{0})d\bx_{0}$. On the other hand,
	\begin{equation}
		\small
		\begin{aligned}
			\nabla^{2}_{\bx_{t}\bx_{t}}K_{t}(\bx_{t}, \bx_{0}) &= \exp\left(-\frac{\|\bx_{t} - \sqrt{\bar{\alpha}_{t}}\bx_{0}\|^{2}}{2(1 - \bar{\alpha}_{t})}\right)\left(\frac{\bx_{t} - \sqrt{\bar{\alpha}_{t}}\bx_{0}}{1 - \bar{\alpha}_{t}}\right)\left(\frac{\bx_{t} - \sqrt{\bar{\alpha}_{t}}\bx_{0}}{1 - \bar{\alpha}_{t}}\right)^{\top} \\
			& + \left(\frac{1}{1 - \bar{\alpha_{t}}}\right)\exp\left(-\frac{\|\bx_{t} - \sqrt{\bar{\alpha}_{t}}\bx_{0}\|^{2}}{2(1 - \bar{\alpha}_{t})}\right)\bI.
		\end{aligned}
	\end{equation}
	Thus for any $\bxi\in\bbR^{d}$ with $\|\bxi\| = 1$, we have 
	\begin{equation}
		\small
		\begin{aligned}
			\sup_{\bxi:\|\bxi\|=1}\bxi^{\top}\nabla^{2}_{\bx_{y}\bx_{t}}K_{t}(\bx_{t}, \bx_{0})\bxi & = \exp\left(-\frac{\|\bx_{t} - \sqrt{\bar{\alpha}_{t}}\bx_{0}\|^{2}}{2(1 - \bar{\alpha}_{t})}\right)\left\|\frac{\bx_{t} - \sqrt{\bar{\alpha}_{t}}\bx_{0}}{1 - \bar{\alpha}_{t}}\right\|^{2} \\
			& + \left(\frac{1}{1 - \bar{\alpha_{t}}}\right)\exp\left(-\frac{\|\bx_{t} - \sqrt{\bar{\alpha}_{t}}\bx_{0}\|^{2}}{2(1 - \bar{\alpha}_{t})}\right) \\
			& \leq \frac{2}{e(1 - \bar{\alpha_{t}})} + \frac{1}{1 - \bar{\alpha_{t}}},
		\end{aligned}
	\end{equation}
	where we use the inequality $axe^{-\frac{ax}{2}} \leq 2e^{-1}$, which proves our second conclusion. 
\end{proof}
The following lemma is an important transformation of conditional expectation which is Tweedie’s Formula \citep{efron2011tweedie}.
\begin{lemma}\label{lem:score}
	Suppose that $\by\mid \bx\sim\cN(\alpha\bx, \beta\bI)$, then $\mE_{\bx\mid \by}[\bx] = \frac{1}{\alpha}(\by + \beta\nabla_{\bv}\log{P}(\bv))$. 
\end{lemma}
\convergenceoriginal*
\begin{proof}
	Due to Lemma \ref{lem:score}, we have 
	\begin{equation}
		\small
		\begin{aligned}
			\mE[\beps_{t}\mid \bx_{t}] & = \mE\left[\frac{1}{\sqrt{1 - \bar{\alpha}_{t}}}\bx_{t} - \frac{\sqrt{\bar{\alpha}_{t}}}{\sqrt{1 - \bar{\alpha}_{t}}}\bx_{0}\mid \bx_{t}\right] \\
			& = \frac{1}{\sqrt{1 - \bar{\alpha}_{t}}}\bx_{t} - \frac{1}{\sqrt{1 - \bar{\alpha}_{t}}}\left(\bx_{t} + (1 - \bar{\alpha}_{t})\nabla_{\bx}\log{P_{t}(\bx_{t})}\right) \\
			& = -\sqrt{1 - \bar{\alpha}_{t}}\nabla_{\bx}\log{P_{t}(\bx_{t})}.
		\end{aligned}
	\end{equation}
	Thus our goal is proving $\beps_{\btheta_{\bS}}(\bx_{t}, t)\overset{P}\longrightarrow -\sqrt{1 - \bar{\alpha}_{t}}\nabla_{\bx}\log{P_{t}(\bx_{t})}$. 
	\begin{equation}\label{eq:score function}
		\small
		\begin{aligned}
			\nabla_{\bx}\log P_{t}(\bx_{t}) &= \nabla_{\bx}P_{t}(\bx_{t})/P_{t}(\bx_{t}) \\
			& = \frac{\nabla_{\bx}\int_{\bbR^{d}}P_{t\mid 0}(\bx_{t}\mid \bx_{0})P_{0}(\bx_{0})d\bx_{0}}{\int_{\bbR^{d}}P_{t\mid 0}(\bx_{t}\mid \bx_{0})P_{0}(\bx_{0})}d\bx_{0} \\
			& = \frac{\mE_{\bx_{0}}[\nabla_{\bx}P_{t\mid 0}(\bx_{t}\mid \bx_{0})]}{\mE_{\bx_{0}}[P_{t\mid 0}(\bx_{t}\mid \bx_{0})]}.
		\end{aligned}
	\end{equation}
	Rewriting the \eqref{eq:empirical optimal eps} as 
	\begin{equation}
		\small
		\begin{aligned}
		\beps_{\btheta_{\bS}}^{*}(\bx_{t}, t) &= -\sqrt{1 - \bar{\alpha}_{t}} \frac{\frac{1}{n}\sum\limits_{i=1}^{n}\left(\frac{1}{2\pi(1 - \bar{\alpha}_{t})}\right)^{\frac{d}{2}}K_{t}(\bx_{t}, \bx^{i}_{0})\left(\frac{\bx_{t} - \sqrt{\bar{\alpha}_{t}}\bx_{0}^{i}}{1 - \bar{\alpha}_{t}}\right)}{\frac{1}{n}\sum\limits_{i=1}^{n}\left(\frac{1}{2\pi(1 - \bar{\alpha}_{t})}\right)^{\frac{d}{2}}K_{t}(\bx_{t}, \bx_{0}^{i})} \\
		& = -\sqrt{1 - \bar{\alpha}_{t}}\frac{\nabla_{\bx} f_{\bS}(\bx_{t})}{f_{\bS}(\bx_{t})},	
		\end{aligned}
	\end{equation}
	where $K_{t}(\bx_{t}, \bx_{0}^{i}) = \exp\left(-\frac{\|\bx - \sqrt{\bar{\alpha}_{t}}\bx^{i}_{0}\|^{2}}{2(1 - \bar{\alpha}_{t})}\right)$. Then, what left is to show that $\frac{\nabla_{\bx} f_{\bS}(\bx_{t})}{f_{\bS}(\bx_{t})}\overset{P}\longrightarrow \nabla_{\bx}\log{P_{t}(\bx_{t})}$. 
	\par
	As can be seen the numerator and denominator of the above equation are respectively empirical estimator of the numerator and denominator of the one in \eqref{eq:score function}. Then, both of them are consistency so that we get the conclusion. To check this, we have 
	\begin{equation}
		\small
		\begin{aligned}
			\mE_{\bS}\left[f_{\bS}(\bx_{t})\right] = \mE_{\bx_{0}}[P_{t\mid 0}(\bx_{t}\mid \bx_{0})] = P_{t}(\bx_{t}).
		\end{aligned}
	\end{equation}
	Note that for any $\bS^{i^{\prime}}$ equals to $\bS$ expected $\bx_{0}^{i^{\prime}}\neq \bx_{0}^{i}$, then for any $D > 0$
	\begin{equation}
		\small
		\begin{aligned}
			\sup_{\bx_{t}: \|\bx_{t}\| < D}\left(f_{\bS}(\bx_{t}) - P_{t}(\bx_{t})\right) & - \sup_{\bx_{t}: \|\bx_{t}\| < D}\left(f_{\bS^{i^{\prime}}}(\bx_{t}) - P_{t}(\bx_{t})\right) \leq \sup_{\bx_{t}: \|\bx_{t}\| < D}(f_{\bS}(\bx_{t}) - f_{\bS^{i^{\prime}}}(\bx_{t})) \\
			& = \frac{1}{n}\left(\frac{1}{2\pi(1 - \bar{\alpha}_{t})}\right)^{\frac{d}{2}}\sup_{\bx_{t}}\left(K_{t}(\bx_{t}, \bx_{0}^{i}) - K_{t}(\bx_{t}, \bx_{0}^{i^{\prime}})\right) \\
			& \leq \frac{1}{n}\left(\frac{1}{2\pi(1 - \bar{\alpha}_{t})}\right)^{\frac{d}{2}}.
		\end{aligned}
	\end{equation}
	Thus by McDiarmid's inequality, we must have 
	\begin{equation}
		\small
		\begin{aligned}
		\bbP\left(\left|\sup_{\bx_{t}:\|\bx_{t}\|\leq D}\left(f_{\bS}(\bx_{t}) - P_{t}(\bx_{t})\right) - \mE\left[\sup_{\bx_{t}:\|\bx_{t}\|\leq D}\left(f_{\bS}(\bx_{t}) - P_{t}(\bx_{t})\right)\right]\right| \geq \epsilon\right) \leq \exp\left(-2N(2\pi(1 - \bar{\alpha}_{t}))^{d}\epsilon^{2}\right).
		\end{aligned}
	\end{equation}
	Thus $\sup_{\bx_{t}:\|\bx_{t}\|\leq D}(f_{\bS}(\bx_{t}) - P_{t}(\bx_{t}))\overset{P}\longrightarrow \mE\left[\sup_{\bx_{t}:\|\bx_{t}\|\leq D}(f_{\bS}(\bx_{t}) - P_{t}(\bx_{t}))\right]$. For any $\bx_{t}, \by_{t}$ with norm smaller than $D$ and $\lambda > 0$, let
	\begin{equation}
		\small
		\begin{aligned}
			D_{j} = \mE\left[f_{\bS}(\bx_{t}) - P_{t}(\bx_{t}) \mid \bx_{0}^{1:j}\right] - \mE\left[f_{\bS}(\bx_{t}) - P_{t}(\bx_{t}) \mid \bx_{0}^{1:j - 1}\right].
		\end{aligned} 
	\end{equation}
	Then $f_{\bS}(\bx_{t}) - P_{t}(\bx_{t}) = \sum_{j=1}^{n}D_{j}$. Let 
	\begin{equation}
		\small
		\begin{aligned}
			U_{j} = \sup_{\bx_{0}^{j}} \mE\left[f_{\bS}(\bx_{t}) - P_{t}(\bx_{t}) \mid \bx_{0}^{1:j}\right] - \mE\left[f_{\bS}(\bx_{t}) - P_{t}(\bx_{t}) \mid \bx_{0}^{1:j - 1}\right]; \\
			L_{j} = \inf_{\bx_{0}^{j}} \mE\left[f_{\bS}(\bx_{t}) - P_{t}(\bx_{t}) \mid \bx_{0}^{1:j}\right] - \mE\left[f_{\bS}(\bx_{t}) - P_{t}(\bx_{t}) \mid \bx_{0}^{1:j - 1}\right],
		\end{aligned}
	\end{equation}
	we have $L_{j}\leq D_{j} \leq U_{j}$. Thus 
	\begin{equation}
		\small
		\begin{aligned}
			\mE_{\bS}\left[\exp\left(\lambda\left[f_{\bS}(\bx_{t}) - P_{t}(\bx_{t})\right]\right)\right] &= \mE_{\bS}\left[\exp\left(\lambda\sum_{j=1}^{n}D_{j}\right)\right] \\
			& = \mE_{\bS}\left[\mE\left[\exp\left(\lambda\sum_{j=1}^{n}D_{j}\right)\mid \bx_{0}^{1:N - 1}\right]\right] \\
			& = \mE_{\bS}\left[\exp\left(\lambda\sum_{j=1}^{N - 1}D_{j}\right)\mE\left[\exp\left(\lambda D_{N}\right)\mid \bx_{0}^{1:N - 1}\right]\right] \\
			& = \prod_{j=1}^{n}\mE_{\bS}\left[\mE\left[\exp\left(\lambda D_{j}\right)\mid \bx_{0}^{1:j - 1}\right]\right] \\
			& \leq \exp\left(\sum\limits_{j=1}^{n}\frac{\lambda^{2}(U_{j} - L_{j})^{2}}{8}\right),
		\end{aligned}
	\end{equation}
	where the last inequality is due to Azuma-Hoeffding's inequality \citep{duchi2016lecture}. On the other hand, we have 
	\begin{equation}
		\small
		\begin{aligned}
			U_{j} - L_{j} & \leq \sup_{\bx_{0}^{i}, \bx_{0}^{i^\prime}}\frac{1}{n}\left[\left(f_{\bx_{0}^{i}}(\bx_{t}) - P_{t}(\bx_{t})\right) - \left(f_{\bx_{0}^{i^\prime}}(\bx_{t}) - P_{t}(\bx_{t})\right)\right] \\
			& \leq \frac{2}{N}\left(\frac{1}{2\pi(1 - \bar{\alpha}_{t})}\right)^{\frac{d}{2}}.
		\end{aligned}
	\end{equation}
	Plugging this into the above equation, we get 
	\begin{equation}\label{eq:sub-gaussian}
		\small
			\mE_{\bS}\left[\exp\left(\lambda\left[f_{\bS}(\bx_{t}) - P_{t}(\bx_{t})\right]\right)\right] \leq \exp\left(\frac{\lambda^{2}}{2n\left(2\pi(1 - \bar{\alpha}_{t})\right)^{d}}\right),
	\end{equation}
	which shows that $f_{\bS}(\bx_{t}) - P_{t}(\bx_{t})$ is a sub-Gaussian process w.r.t. $\bx_{t}$. 
	\par
	Due to $\bx_{t}$ has bounded norm, there exists a $\delta$-cover $\cC(\delta, D)$ of $l_{2}$-ball with radius $D$ such that for any $\bx_{t}$ there exists $\by_{t}\in \cC(\delta, D)$ with $\|\bx_{t} - \by_{t}\|\leq \delta$. Due to Lemma \ref{lem:continuity}
	\begin{equation}
		\small
		\begin{aligned}
			\mE_{\bS}\left[\sup_{\bx_{t}}(f_{\bS}(\bx_{t}) - P_{t}(\bx_{t}))\right] &= \mE_{\bS}\left[\sup_{\substack{\bx_{t}, \by_{t};\\ \|\bx_{t} - \by_{t}\|\leq \delta}}(f_{\bS}(\bx_{t}) - P_{t}(\bx_{t})) - (f_{\bS}(\by_{t}) - P_{t}(\by_{t}))\right]\\
			& + \mE\left[\sup_{\by_{t}\in\cC(\delta, D)}(f_{\bS}(\by_{t}) - P_{t}(\by_{t}))\right] \\
			& \leq 2\delta\sqrt{\frac{1}{e(1 - \bar{\alpha}_{t})}} + \sqrt{\frac{2\log{|\cC(\delta, D)}|}{n\left(2\pi(1 - \bar{\alpha}_{t})\right)^{d}}},
		\end{aligned}
	\end{equation} 	
	where the last inequality is due to \eqref{eq:sub-gaussian} and Exercise 3.7 in \citep{duchi2016lecture}. Due to the arbitrarity of $\delta$ and taking $n\to \infty$, we get $\mE_{\bS}\left[\sup_{\bx_{t}}(f_{\bS}(\bx_{t}) - P_{t}(\bx_{t}))\right]\longrightarrow 0$, which implies $f_{\bS}(\bx_{t})\overset{P}\longrightarrow P_{t}(\bx_{t}$ for any $\bx_{t}$. 
	\par
	Thus we show that denominator of \eqref{eq:empirical optimal eps} converge to the one of \eqref{eq:score function} in probability. Similarly, we can prove the numerator of \eqref{eq:empirical optimal eps} converge to the one of \eqref{eq:score function} in probability. First, we have 
	\begin{equation}
		\small
		\begin{aligned}
			\mE_{\bS}\left[\nabla_{\bx} f_{\bS}(\bx_{t})\right]= \mE_{\bx_{0}}[\nabla_{\bx}P_{t\mid 0}(\bx_{t}\mid \bx_{0})] = \nabla_{\bx}P_{t}(\bx_{t}). 
		\end{aligned}
	\end{equation}
	Then,
	\begin{equation}
		\small
		\begin{aligned}
			& \sup_{\bx_{t}: \|\bx_{t}\| < D}\left\|\nabla_{\bx}f_{\bS}(\bx_{t}) - \nabla_{\bx}P_{t}(\bx_{t})\right\|  - \sup_{\bx_{t}: \|\bx_{t}\| < D}\left\|\nabla_{\bx}f_{\bS^{i^{\prime}}}(\bx_{t}) - \nabla_{\bx}P_{t}(\bx_{t})\right\| \\
			& \leq \sup_{\bx_{t}: \|\bx_{t}\| < D}\left\|\nabla_{\bx}f_{\bS}(\bx_{t}) - \nabla_{\bx}f_{\bS^{i^{\prime}}}(\bx_{t})\right\| \\
			& = \frac{1}{n}\left(\frac{1}{2\pi(1 - \bar{\alpha}_{t})}\right)^{\frac{d}{2}}\sup_{\bx_{t}}\left\|K_{t}(\bx_{t}, \bx_{0}^{i})\left(\frac{\bx_{t} - \sqrt{\bar{\alpha}_{t}}\bx_{0}^{i}}{1 - \bar{\alpha}_{t}}\right) - K_{t}(\bx_{t}, \bx_{0}^{i^{\prime}})\left(\frac{\bx_{t} - \sqrt{\bar{\alpha}_{t}}\bx_{0}^{i^{\prime}}}{1 - \bar{\alpha}_{t}}\right)\right\| \\
			& \leq \frac{1}{n}\left(\frac{1}{2\pi(1 - \bar{\alpha}_{t})}\right)^{\frac{d}{2}}\sup_{\bx_{t}, \bx_{0}}K_{t}(\bx_{t}, \bx_{0})\left\|\frac{\bx_{t} - \sqrt{\bar{\alpha}_{t}}\bx_{0}^{i^{\prime}}}{1 - \bar{\alpha}_{t}}\right\| \\
			& \leq \frac{1}{n}\left(\frac{1}{2\pi(1 - \bar{\alpha}_{t})}\right)^{\frac{d}{2}}\sqrt{\frac{1}{e(1 - \bar{\alpha}_{t})}},
		\end{aligned}
	\end{equation}
	where the last inequality is due to $axe^{-\frac{ax^{2}}{2}} \leq \sqrt{a / e}$. Thus by McDiarmid's inequality, we must have 
	\begin{equation}
		\small
		\begin{aligned}
			& \bbP\left(\left|\sup_{\bx_{t}:\|\bx_{t}\|\leq D}\left\|\nabla_{\bx}f_{\bS}(\bx_{t}) - \nabla_{\bx}P_{t}(\bx_{t})\right\| - \mE\left[\sup_{\bx_{t}:\|\bx_{t}\|\leq D}\left\|\nabla_{\bx}f_{\bS}(\bx_{t}) - \nabla_{\bx}P_{t}(\bx_{t})\right\|\right]\right| \geq \epsilon\right) \\
			& \leq \exp\left(-2en(2\pi(1 - \bar{\alpha}_{t}))^{d}(1 - \bar{\alpha}_{t})\epsilon^{2}\right).
		\end{aligned}
	\end{equation}
	Then we show that $\sup_{\bx_{t}:\|\bx_{t}\|\leq D}\left\|\nabla_{\bx}f_{\bS}(\bx_{t}) - \nabla_{\bx}P_{t}(\bx_{t})\right\|$ converge to its expectation in probability. What left is showing its expectation converges to zero. Similar to the proof of \eqref{eq:sub-gaussian}, we can prove 
	\begin{equation}
		\small
			\mE_{\bS}\left[\exp\left(\lambda\left[\|\nabla_{\bx_{t}}f_{\bS}(\bx_{t}) - \nabla_{\bx_{t}}P_{t}(\bx_{t})\right]\right)\right] \leq \exp\left(\sum\limits_{j=1}^{n}\frac{\lambda^{2}\|\nabla_{\bx_{t}}U_{j} - \nabla_{\bx_{t}}L_{j}\|^{2}}{8}\right).	
	\end{equation} 
	On the other hand, due to Lemma \ref{lem:continuity}, we have 
	\begin{equation}
		\small
		\begin{aligned}
			\|\nabla_{\bx_{t}}U_{j} - \nabla_{\bx_{t}}L_{j}\| & \leq \sup_{\bx_{0}^{i}, \bx_{0}^{i^\prime}}\frac{1}{n}\left\|\left[\left(\nabla_{\bx_{t}}f_{\bx_{0}^{i}}(\bx_{t}) - \nabla_{\bx_{t}}P_{t}(\bx_{t})\right) - \left(\nabla_{\bx_{t}}f_{\bx_{0}^{i^\prime}}(\bx_{t}) - \nabla_{\bx_{t}}P_{t}(\bx_{t})\right)\right]\right\| \\
			& \leq \frac{4}{N}\left(\frac{1}{2\pi(1 - \bar{\alpha}_{t})}\right)^{\frac{d}{2}}\sqrt{\frac{1}{e(1 - \bar{\alpha}_{t})}},
		\end{aligned}
	\end{equation}
	which implies 
	\begin{equation}
		\small
			\mE_{\bS}\left[\exp\left(\lambda\left[\|\nabla_{\bx_{t}}f_{\bS}(\bx_{t}) - \nabla_{\bx_{t}}P_{t}(\bx_{t})\right]\right)\right] \leq\exp\left(\frac{2\lambda^{2}}{en(2\pi(1 - \bar{\alpha}_{t}))^{\frac{d}{2}}(1 - \bar{\alpha}_{t})}\right).
	\end{equation}
	Thus, due to Lemma \ref{lem:continuity},
	\begin{equation}
		\small
		\begin{aligned}
			& \mE_{\bS}\left[\sup_{\bx_{t}}\|\nabla_{\bx_{t}}f_{\bS}(\bx_{t}) - \nabla_{\bx_{t}}P_{t}(\bx_{t})\|\right] \\
			&\leq  \mE_{\bS}\left[\sup_{\substack{\bx_{t}, \by_{t};\\ \|\bx_{t} - \by_{t}\|\leq \delta}}\|(\nabla_{\bx_{t}}f_{\bS}(\bx_{t}) - \nabla_{\bx_{t}}P_{t}(\bx_{t})) - (\nabla_{\by_{t}}f_{\bS}(\by_{t}) - \nabla_{\by_{t}}P_{t}(\by_{t}))\|\right]\\
			& + \mE\left[\sup_{\by_{t}\in\cC(\delta, D)}\|\nabla_{\by_{t}}f_{\bS}(\by_{t}) - \nabla_{\by_{t}}P_{t}(\by_{t})\|\right] \\
			& \leq 2\delta\left(\frac{1}{2\pi(1 - \bar{\alpha}_{t})}\right)^{\frac{d}{2}}\left(\frac{2 + e}{e(1 - \bar{\alpha_{t}})}\right) + \sqrt{\frac{8\log{|\cC(\delta, D)}|}{en\left(2\pi(1 - \bar{\alpha}_{t})\right)^{d}(1 - \bar{\alpha}_{t})}},
		\end{aligned}
	\end{equation}
	By taking a proper $\delta$ and $n\to \infty$, we show that $\mE_{\bS}\left[\sup_{\bx_{t}}\|\nabla_{\bx_{t}}f_{\bS}(\bx_{t}) - \nabla_{\bx_{t}}P_{t}(\bx_{t})\|\right]$ converges to zero. Thus, the denominator and numerator of \eqref{eq:empirical optimal eps} are respectively converge to the ones of \eqref{eq:score function}. 
	\par
	Finally, due to $\|\bx_{t}\|$ is bounded, $P_{t}(\bx)$ and $\nabla_{\bx_{t}}P_{t}(\bx_{t})$ are all continuous, $\nabla_{\bx_{t}}\log{P_{t}(\bx_{t})}$ is continuous. Thus by Slutsky's theorem \citep{shiryaev2016probability}, we prove our result. 
\end{proof}
\section{Proofs in Section \ref{sec:estimating forward}}\label{app:proofs of section estimating forward}
\scoreequivalent*
\begin{proof}
	Due to \eqref{eq:forward process} and Tweedie’s formula \citep{efron2011tweedie}, we know
	\begin{equation}
		\small
			\mE[\bxi_{t, s}\mid \bx_{t}] = \frac{1}{\sqrt{1 - r_{t, s}}}\bx_{t} - \frac{1}{\sqrt{1 - r_{t, s}}}\left(\bx_{t} + (1 - r_{t, s})\nabla_{\bx}\log{P_{t}}(\bx_{t})\right) = -\sqrt{1 - r_{t, s}}\nabla_{\bx}\log{P_{t}}(\bx_{t}). 
	\end{equation}
	Thus $\mE[\bxi_{t, s}\mid \bx_{t}] / \sqrt{1 - r_{t, s}}$ is invariant w.r.t. $s$, which verifies our conclusion. 
\end{proof}
\ourempiricalsolution*
\begin{proof}
	Due to \eqref{eq:forward process}, for any $t$, our training objective \eqref{eq:our objective} can be written as 
	\begin{equation}
		\small
		\begin{aligned}
			& \inf_{\bxi_{\btheta}}\mE_{s}\left[\frac{1}{n}\sum_{i=1}^{n}\mE_{\bxi_{t, s}}\left[\left\|\frac{\bxi_{t, s}}{\sqrt{1 - r_{t, s}}} - \bxi_{\btheta}(\sqrt{r_{t, s}}\bx_{s}^{i} + \sqrt{1 - r_{t, s}}\bxi_{t, s}, t)\right\|^{2}\right]\right] \\
			& = \inf_{\bxi_{\btheta}}\mE_{s}\left[\frac{1}{n}\sum_{i=1}^{n}\int_{\bbR^{d}}\left\|\frac{\bxi_{t, s}}{\sqrt{1 - r_{t, s}}} - \bxi_{\btheta}(\sqrt{r_{t, s}}\bx_{s}^{i} + \sqrt{1 - r_{t, s}}\bxi_{t, s}, t)\right\|^{2}\left(\frac{1}{2\pi}\right)^{\frac{d}{2}}\exp\left(\frac{-\|\bxi_{t, s}\|^{2}}{2}\right)d\bxi_{t, s}\right] \\
			& = \inf_{\bxi_{\btheta}}\mE_{s}\left[\frac{1}{n}\sum_{i=1}^{n}\int_{\bbR^{d}}\left\|\bxi_{\btheta}(\bx, t) - \frac{\bx - \sqrt{r_{t, s}}\bx_{s}^{i}}{1 - r_{t, s}}\right\|^{2}\left(\frac{1}{2\pi(1 - r_{t, s})}\right)^{\frac{d}{2}}\exp\left(-\frac{\left\|\bx - \sqrt{r_{t, s}}\bx_{s}^{i}\right\|^{2}}{2(1 - r_{t, s})}\right)d\bx\right].
		\end{aligned}
	\end{equation}
	Then following the proof of Theorem \ref{pro:generalization of ddpm}, we prove our conclusion. 
\end{proof}
As we have clarified in the mainbody of this paper, the objective \eqref{eq:our objective population} has global minima $\mE[\bxi_{t, t - 1}\mid \bx_{t}]$. We formally prove this conclusion in the following lemma. 
\begin{lemma}
	For $\bxi_{\btheta}(\cdot, \cdot)$ with enough functional capacity, the problem \eqref{eq:our objective population} has global minima $\bxi_{\btheta}(\bx_{t}, t) = \mE[\bxi_{t, t - 1}\mid \bx_{t}] / \sqrt{1 - r_{t, t - 1}}$.
\end{lemma}
\begin{proof}
	For any specific $t$, due to the optimality of conditional expectation of minimizing min-square estimation \citep{banerjee2005optimality}, 
	\begin{equation}
		\small
		\begin{aligned}
		& \inf_{\bxi_{\btheta(\cdot, t)}}\mE_{s}\left[\mE_{\bx_{s}, \bxi_{t, s}}\left[\left\|\frac{\bxi_{t, s}}{\sqrt{1 - r_{t, s}}} - \bxi_{\btheta}(\sqrt{r_{t, s}}\bx_{s} + \sqrt{1 - r_{t, s}}\bxi_{t, s}, t)\right\|^{2}\right]\right] \\
		& \geq \mE_{s}\left[\inf_{\bxi_{\btheta(\cdot, t)}}\mE_{\bx_{s}, \bxi_{t, s}}\left[\left\|\frac{\bxi_{t, s}}{\sqrt{1 - r_{t, s}}} - \bxi_{\btheta}(\sqrt{r_{t, s}}\bx_{s} + \sqrt{1 - r_{t, s}}\bxi_{t, s}, t)\right\|^{2}\right]\right] \\
		& = \mE_{s}\left[\mE_{\bx_{s}, \bxi_{t, s}}\left[\left\|\frac{\bxi_{t, s}}{\sqrt{1 - r_{t, s}}} - \mE\left[\frac{\bxi_{t, s}}{\sqrt{1 - r_{t, s}}}\mid \bx_{t}\right]\right\|^{2}\right]\right],
		\end{aligned}
	\end{equation}
	where the first inequality becomes when $\bxi_{\btheta}(\sqrt{r_{t, s}}\bx_{s} + \sqrt{1 - r_{t, s}}\bxi_{t, s}, t) = \mE\left[\frac{\bxi_{t, s}}{\sqrt{1 - r_{t, s}}}\mid \bx_{t}\right]$, which is invariant w.r.t. s due to Lemma \ref{lem:equivalence}.  
\end{proof}
\par
It worth noting that our training objective is another view of score matching \citep{song2020score}, which approximate score function  $\nabla_{\bx_{t}}\log{P_{t}}(\bx_{t})$. Then using the approximated $\nabla_{\bx_{t}}\log{P_{t}}(\bx_{t})$ to running a reverse-time stochastic differential equation to generate data. In this regime, they leverage a model $\bs_{\btheta}(\bx_{t}, t)$ to minimizing $\mE_{\bx_{t}}[\|\bs_{\btheta}(\bx_{t}, t) - \nabla_{\bx_{t}}\log{P_{t}}(\bx_{t})\|^{2}]$ to get the approximated score function $\bs_{\btheta}(\bx_{t}, t)$. It has been proven in \citep{vincent2011connection} that for any $s < t$, it holds 
\begin{equation}
	\small
	\inf_{\bs_{\btheta}} \mE_{\bx_{t}}\left[\left\|\bs_{\btheta}(\bx_{t}, t) - \nabla_{\bx_{t}}\log{P_{t}}(\bx_{t})\right\|^{2}\right] = \inf_{\bs_{\btheta}}\mE_{\bx_{s}}\left[\mE_{\bx_{t}\mid \bx_{s}}\left[\left\|\bs_{\btheta}(\bx_{t}, t) - \nabla_{\bx_{t}}\log{P_{t\mid s}}(\bx_{t}\mid \bx_{s})\right\|^{2}\right]\right].
\end{equation}
Due to \eqref{eq:forward process}, we know that our training objective \eqref{eq:our objective population} is equivalent to 
\begin{equation}\small
	\inf_{\bs_{\btheta}}\mE_{s}\left[\mE_{\bx_{s}}\left[\mE_{\bx_{t}\mid \bx_{s}}\left[\left\|\bs_{\btheta}(\bx_{t}, t) - \nabla_{\bx_{t}}\log{P_{t\mid s}}(\bx_{t}\mid \bx_{s})\right\|^{2}\right]\right]\right],
\end{equation}
which is equivalent to score-matching as in \citep{song2020score} but with a random initial time step $s$ (the $s$ in \citep{song2020score} is fixed as zero). 
\par
\convergenceofourobjective*
\begin{proof}
	 For any given $t$ and $s < t$, we know
	 \begin{equation}
	 	\small
	 	\begin{aligned}
	 		P_{t\mid s}(\bx_{t}\mid \bx_{s}) &= \left(\frac{1}{2\pi(1 - r_{t, s})}\right)^{\frac{d}{2}}\exp\left(-\frac{\left\|\bx - \sqrt{r_{t, s}}\bx_{s}\right\|^{2}}{2(1 - r_{t, s})}\right); \\
	 		\nabla_{\bx_{t}}P_{t\mid s}(\bx_{t}\mid \bx_{s}) &= \left(\frac{1}{2\pi(1 - r_{t, s})}\right)^{\frac{d}{2}}\exp\left(-\frac{\|\bx - \sqrt{r_{t, s}}\bx_{s}^{i}\|^{2}}{2(1 - r_{t, s})}\right)\left(\frac{\bx - \sqrt{r_{t, s}}\bx_{s}}{1 - r_{t,s}}\right). 
	 	\end{aligned}
 		\end{equation}
 	Thus, the numerator and denominator are respectively unbiased estimator to $\nabla_{\bx_{t}}P_{t}(\bx_{t})$ and $P_{t}(\bx_{t})$. Since $s$ is finite, and $\mE[\bxi_{t, t - 1}\mid \bx_{t}] / \sqrt{1 - r_{t, t - 1}} = \nabla\log{P_{t}(\bx_{t})}$, we can similarly prove our result as in Theorem \ref{thm:sample complexity original}. 
\end{proof}
\section{Extra Experiments}
In this section, we present some of generated data by different diffusion models in the main part of this paper.  
\subsection{Accumulated Optimization Bias Improves Generalization}\label{app:accu}
\begin{figure}[t]\centering
	\includegraphics[width=1\textwidth]{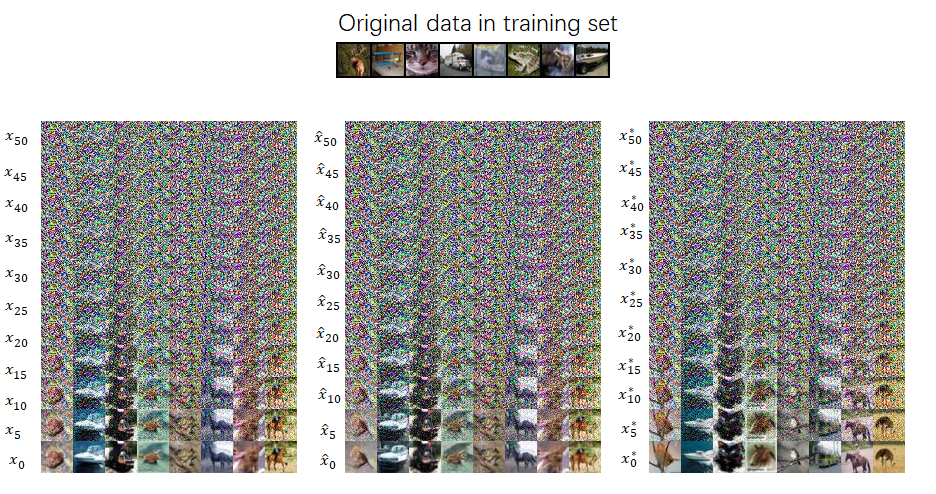}
	\caption{The generated \texttt{CIFAR10}, starting with noisy data constructed by training set. From the left to right are respectively the data generated by diffusion models trained by \eqref{eq:our objective} and \eqref{eq:empirical objective}.}
	\label{fig:cifar10 starting}
\end{figure}
\begin{figure}[t]\centering
	\includegraphics[width=1\textwidth]{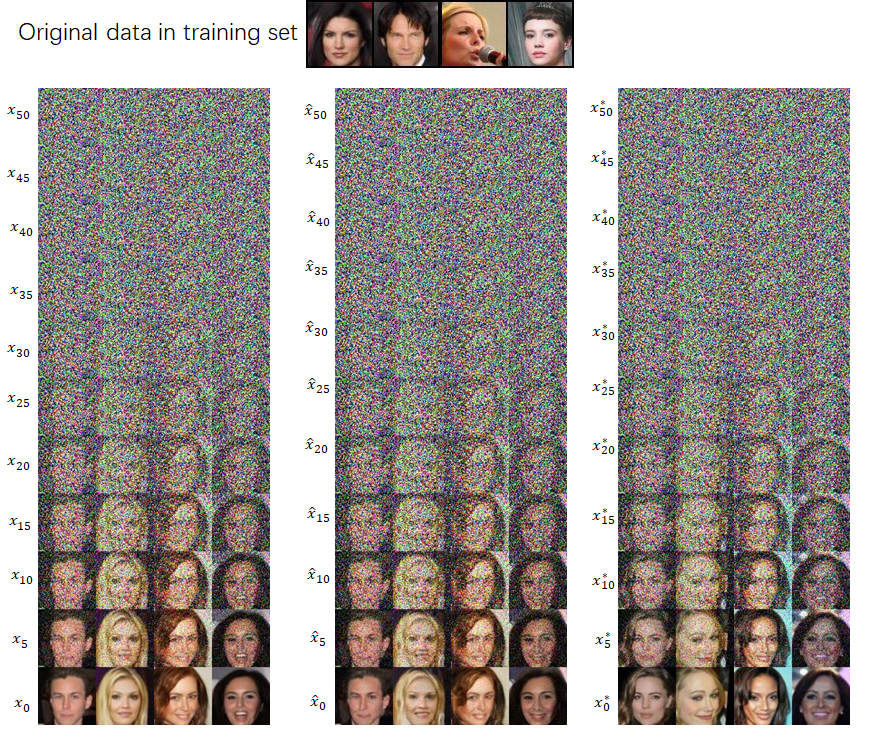}
	\caption{The generated \texttt{CelebA}, starting with noisy data constructed by training set. From the left to right are respectively the data generated by diffusion models trained by \eqref{eq:our objective} and \eqref{eq:empirical objective}.}
	\label{fig:celeba starting}
\end{figure}
In section \ref{sec:The Optimization Bias Helps Extrapolating}, we have verified that the generated $\bx_{0}$ can exist in training set when starting from $\bx_{15}$ generated by data in training set. However, we claim that when accumulating enough bias during the reverse process of generating, the generalization problem can be obviated. That says we start the reverse process from the same $\bx_{50}$ generated by the data in training set. The results are in Figure \ref{fig:cifar10 starting} and \ref{fig:celeba starting}. As can be seen, the generated data do not visually similar to the original training data. 
\par
Finally, as we have claim in Section \ref{sec:The Optimization Bias Helps Extrapolating}, the optimization bias enables the trained diffusion model to obviate generate data existed in the training set. For each generated data, we verifies it by searching the nearest data in the training set. Some of generated data are in Figure \ref{fig:cifar10 nearest} and \ref{fig:celeba nearest}.  
\begin{figure}[t]\centering
	\includegraphics[width=1\textwidth]{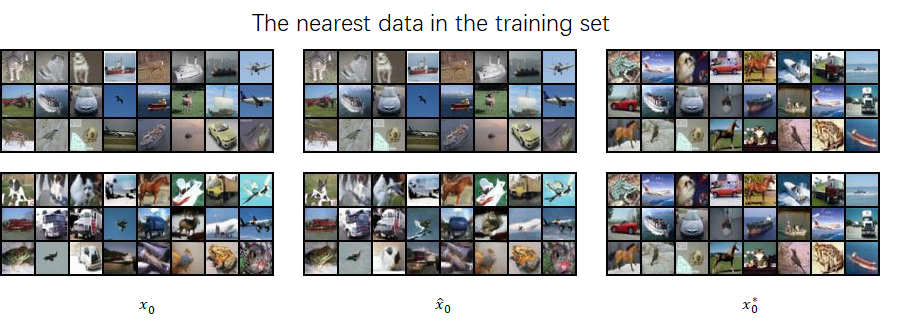}
	\caption{The generated \texttt{CIFAR10}, the bottom and top line are respectively the generated data and the its nearest data in the training set. From the left to right are respectively the data generated by diffusion models trained by \eqref{eq:empirical objective}, \eqref{eq:empirical objective} and empirical optima.}
	\label{fig:cifar10 nearest}
\end{figure}
\begin{figure}[t]\centering
	\includegraphics[width=1\textwidth]{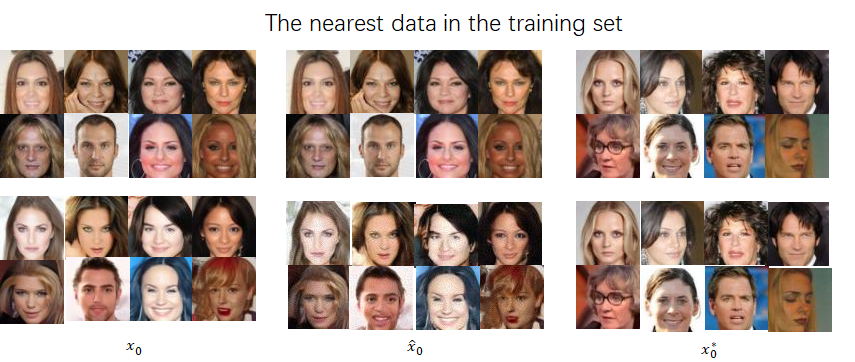}
	\caption{The generated \texttt{CelebA}, the bottom and top line are respectively the generated data and the its nearest data in the training set. From the left to right are respectively the data generated by diffusion models trained by \eqref{eq:empirical objective}, \eqref{eq:empirical objective} and empirical optima.}
	\label{fig:celeba nearest}
\end{figure}
\subsection{Data Generated by Different Diffusion Models}\label{app:Data Generated by Empirical Optima}
We present batch of generated data by different diffusion models, i.e., the ones trained by \eqref{eq:empirical objective}, \eqref{eq:our objective} and the empirical optima. They are respectively represented by $\bx_{t}$, $\hat{\bx}_{t}$, and $\bx_{t}^{*}$, and staring with the same standard Gaussian noise. Similar to Section \ref{sec:optimization bias}, the data are generated by 50 steps DDIM \citep{song2020denoising}. The \texttt{CIFAR10} and \texttt{CelebA} are respectively in Figure \ref{fig:cifar10 gen all} and \ref{fig:celeba gen all}. As can be seen, the $\bx_{t}$ and $\hat{\bx}_{t}$ are close to each other, while $\hat{\bx}_{t}$ is noisy than $\bx_{t}$. This further verifies there is a trade-off between generalization and optimization as we discussed in the Section \ref{sec:exp}. 

\begin{figure}[t]\centering
	\includegraphics[width=1\textwidth]{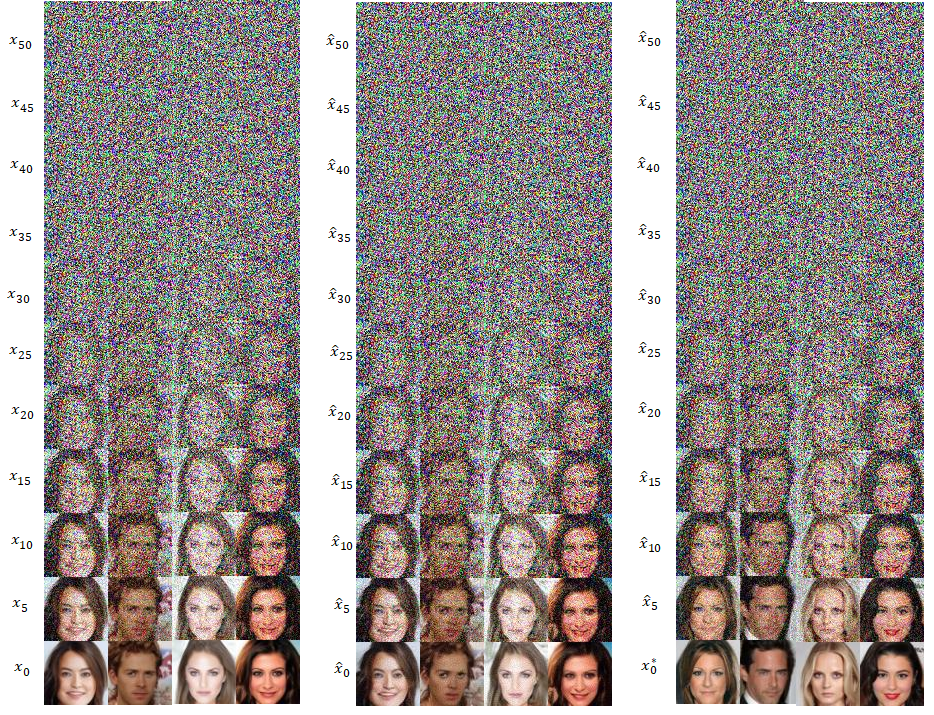}
	\caption{The generated \texttt{CelebA}, from the left to right are respectively the data generated by diffusion models trained by \eqref{eq:empirical objective}, \eqref{eq:our objective} and the empirical optima \eqref{eq:empirical optimal eps}.}
	\label{fig:celeba gen all}
\end{figure} 
 
\begin{figure*}[t!]\centering
	\subfloat[\texttt{CIFAR10} generated by the model trained by \eqref{eq:empirical objective}.]{
		\includegraphics[width=1.0\textwidth]{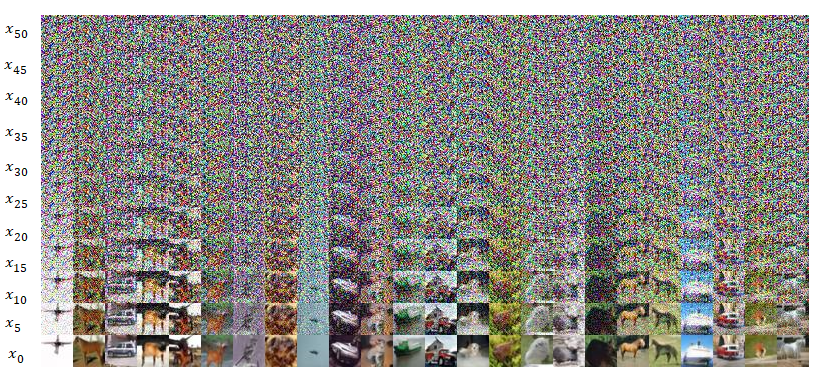}} \\
	\subfloat[\texttt{CIFAR10} generated by the model trained by \eqref{eq:our empirical solution}.]{
		\includegraphics[width=1.0\textwidth]{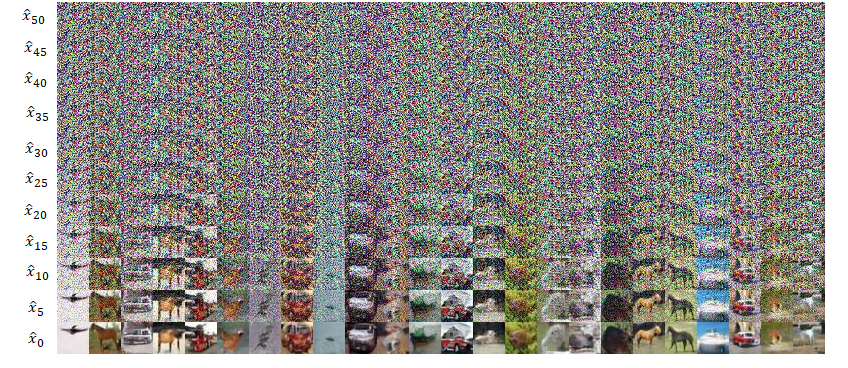}} \\
	\subfloat[\texttt{CIFAR10} generated by the model trained empirical optima \eqref{eq:empirical optimal eps}.]{
		\includegraphics[width=1.0\textwidth]{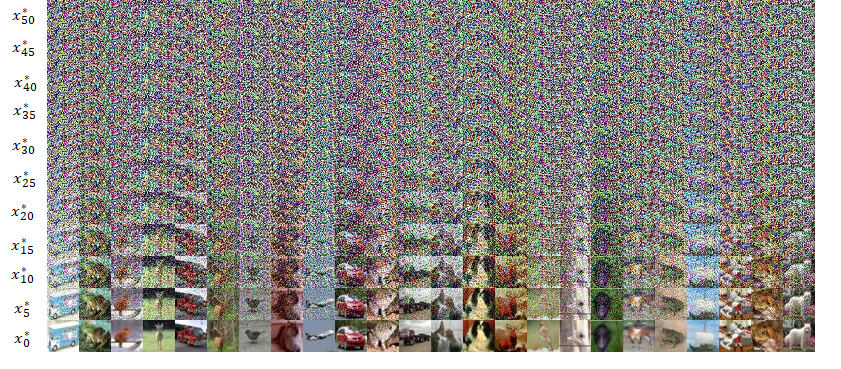}}
	\caption{The generated \texttt{CIFAR10}
	}
	\label{fig:cifar10 gen all}
\end{figure*}
\subsection{Generated $\hat{\bx_{0}}$}
In this subsection, we compare some data generated by the diffusion model trained by \eqref{eq:our objective} and \eqref{eq:empirical objective}. Though the first model has no potential generalization problem, its generated data are noisy compared with $\bx_{t}$. The data are in Figure \ref{fig:cifar10 comp} and \ref{fig:celeba comp}. 
\begin{figure}[t]\centering
	\includegraphics[width=1\textwidth]{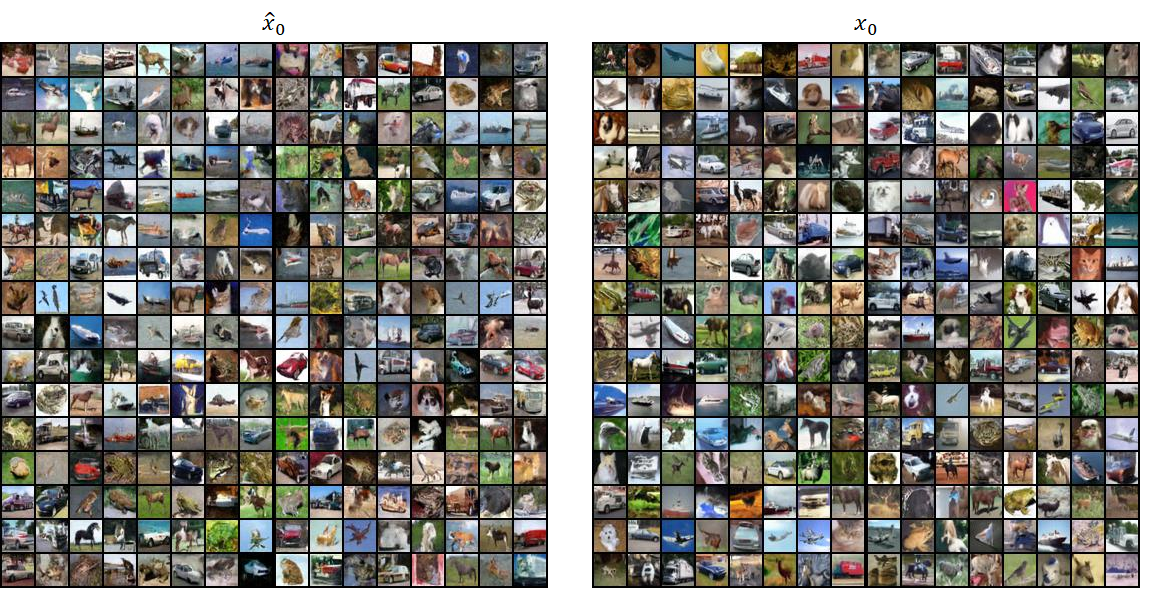}
	\caption{The generated \texttt{CIFAR10}, from the left to right are respectively the data generated by diffusion models trained by \eqref{eq:our objective} and \eqref{eq:empirical objective}.}
	\label{fig:cifar10 comp}
\end{figure}
\begin{figure}[t]\centering
	\includegraphics[width=1\textwidth]{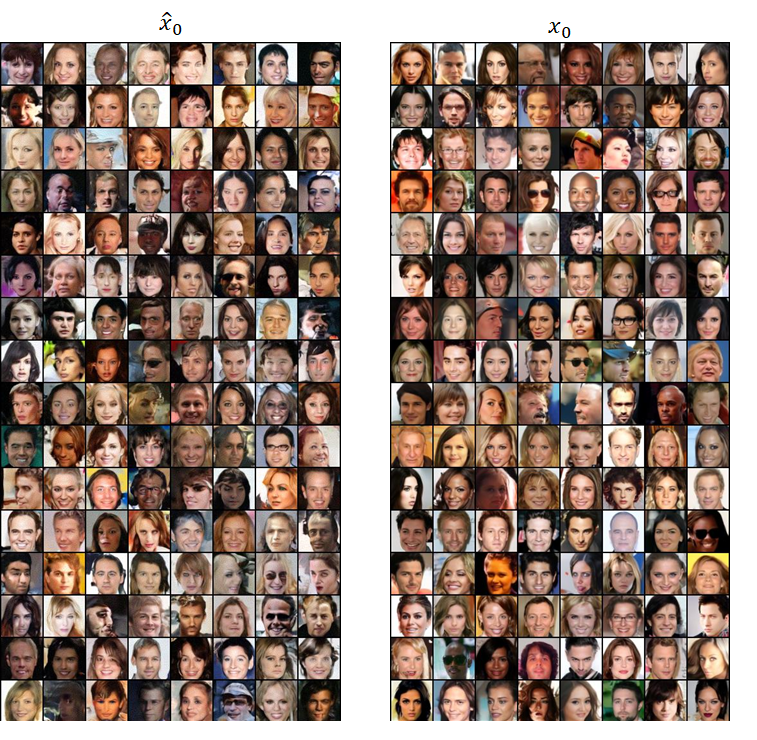}
	\caption{The generated \texttt{CelebA}, from the left to right are respectively the data generated by diffusion models trained by \eqref{eq:our objective} and \eqref{eq:empirical objective}.}
	\label{fig:celeba comp}
\end{figure}
 